\documentclass{article} 
\usepackage{times,fullpage}
\usepackage{graphicx,authblk}
\usepackage{algorithm,appendix}
\usepackage{algorithmic}
\usepackage{natbib}
\usepackage{amsfonts}
\usepackage{url}
\usepackage{multirow}
\usepackage{algorithm,algorithmic,amsmath,amsthm,amsfonts}
\usepackage{verbatim,graphicx,subfigure}
\usepackage{mathtools}
\usepackage[usenames,dvipsnames,svgnames,table]{xcolor}
\newcommand{\BlackBox}{\rule{1.5ex}{1.5ex}}  
\newtheorem{assumption}{Assumption} 
\newtheorem*{assumption*}{Assumption}
 
\newtheorem{thm}{Theorem}
\newtheorem{lemma}{Lemma}

\newtheorem{defn}{Definition}

\def\bbE{\mathbb{E}}

\def\bbP{\mathbb{P}}
\def\bbR{\mathbb{R}}

\def\cH{\mathcal{H}}
\def\cX{\mathcal{X}}
\def\cC{\mathcal{C}}
\def\cO{\mathcal{O}}
\def\cP{\mathcal{P}}
\def\cB{\mathcal{B}}

\def\cS{\mathcal{S}}
\def\cR{\mathcal{R}}

\def\cF{\mathcal{F}}
\def\la{\langle}
\def\ra{\rangle}

\def\nabf{\nabla_{\theta}f}

\def\nabR{\nabla \cR}

\def\Rbetastar{\cR_{\beta}^{*}}

\def\pi{p_{i}}

\def\thetahatT{\hat{\theta}_T}

\def\cR{\mathcal{R}}
\def\thetatminus{\theta_{t-1}}

\def\bhatT{\hat{b}_T}

\newcommand{\defeq}{\mbox{$\;\stackrel{\mbox{\tiny\rm def}}{=}\;$}}
\newcommand{\indep}{{\bot\negthickspace\negthickspace\bot}}
\newcommand{\leqa}{\mbox{$\;\stackrel{\mbox{\tiny\rm (a)}}{\leq}\;$}}
\newcommand{\leqb}{\mbox{$\;\stackrel{\mbox{\tiny\rm (b)}}{\leq}\;$}}

\DeclareMathOperator{\sgn}{sgn}
\DeclarePairedDelimiter\abs{\lvert}{\rvert}%
\DeclarePairedDelimiter\norm{\lVert}{\rVert}%

\makeatletter
\let\oldabs\abs
\def\abs{\@ifstar{\oldabs}{\oldabs*}}
\let\oldnorm\norm
\def\norm{\@ifstar{\oldnorm}{\oldnorm*}}
\makeatother

\begin{document}
\title{Active Model Aggregation via Stochastic Mirror Descent}
\date{}
\author{Ravi Ganti\\gantimahapat@wisc.edu\thanks{ Most of the work was done while the author was at Georgia Institute of Technology, Atlanta, GA, USA.}}
\affil{Wisconsin Institutes for Discovery,\\ 330 N Orchard St, Madison, WI, 53715}
\maketitle
\begin{abstract}
We consider the problem of learning convex aggregation of models, that is as good as the best convex aggregation, for the binary classification problem. Working in the stream based active learning setting, where the active learner has to make a decision on-the-fly, if it wants to query for the label of the point currently seen in the stream, we propose a stochastic-mirror descent algorithm, called SMD-AMA, with entropy regularization. We establish an excess risk bounds for the loss of the convex aggregate returned by SMD-AMA to be of the order of $O\left(\sqrt{\frac{\log(M)}{{T^{1-\mu}}}}\right)$, where $\mu\in [0,1)$ is an algorithm dependent parameter, that trades-off the number of labels queried, and excess risk. 
We demonstrate experimental results on standard UCI datasets, and show that, compared to a passive learning, SMD-AMA queries 10-68\% of the points, to get to similar accuracy as a passive learner.
\end{abstract}
\section{Introduction}
\label{sec:agg_intro}
Machine learning techniques have become popular in many fields such as Astronomy, Physics, Biology, Chemistry, Web Search, Finance. With the availability of large amounts of data, and greater computational power, newer machine learning algorithms, and applications have been discovered. A very popular subclass of machine learning problems fall under the category of supervised learning problems. Classification and regression are two most popular problems in supervised learning, where the learner is provided with labeled data, and is required to predict the labels of unseen points. In the case of classification problems, these labels are discrete, whereas in the case of regression problems, these labels are continuous.

Supervised learning critically relies on the presence of labeled data. The cost of obtaining labels for different data points depends on the problem domain. For example, in Astronomy it is easy to get access to tons and tons of unlabeled data.  However, obtaining labeled data is usually hard. In problems of speech recognition, information extraction, obtaining labeled data is often tedious and requires domain expertise~\citep{zhu2005ssl,settles2008active,settlestr09}. In such cases, a natural question that arises is that can we learn with limited supervision?

In the classical passive, binary classification problem one has access to labeled samples $\cS=\{(x_1,y_1),\ldots,(x_T,y_T)\}$, drawn from an unknown distribution $P$ defined on a domain $\cX\times \{-1,+1\}$, where $\cX\subset \bbR^d$. The points $\{x_1,\ldots,x_T\}$ are sampled i.i.d. from the marginal distribution $P_{\cX}$, and the labels $y_1,\ldots,y_T$ are sampled from the conditional distribution $P_{Y|X=x}$. Algorithms such as SVMs, logistic regression choose a hypothesis class $\cH$, and an appropriate loss function $L(\cdot)$, and solve some sort of an empirical risk minimization problem to return a hypothesis $\hat{h}\in\cH$, whose risk, $R(\hat{h})\defeq \bbE_{x,y\sim P} L(y\hat{h}(x))$ is small. 
Given a collection of models $\cB=\{b_1,\ldots,b_M\}$, the collection of convex aggregations of these models is a much richer class. Based on this idea many algorithms in machine learning and in approximation theory have been proposed which learn an aggregation of basic models. Ensemble methods~\citep{dietterich2000ensemble,schapire2003boosting} are methods that combine a large number of simple models to learn a single powerful model. Boosting algorithms such as AdaBoost~\citep{freund1995desicion}, and LogitBoost~\citep{friedman2000additive} can be viewed as performing aggregation via functional gradient descent~\citep{mason1999functional}. The key idea here is minimization of a convex loss, exponential loss in the case of AdaBoost, and logistic loss in the case of LogitBoost, via a sequential aggregation of models in $\cF$. In the case of boosting, the set of base models are weak learners, and by aggregating the models we aim to  boost the learning capabilities of the final aggregated model. The problem of model aggregation for regression models  was first proposed by~\citet{nemirovski2000topics}. Given a collection of models $\cB=\{b_1,\ldots,b_M\}$,~\citet{nemirovski2000topics} outlined three problems of model aggregation, namely, model selection, convex aggregation and linear aggregation. \textbf{In this paper we are interested in actively learning a convex aggregation of models for the binary classification problem}. Let $b(x)\defeq [b_1(x),\ldots,b_M(x)]\in\{-1,+1\}^M$, and $\la \cdot,\cdot\ra$ denote the Euclidean dot product. Given a convex, margin based loss function $L:\bbR\rightarrow\bbR_{+}$, we want a procedure that outputs a model, $f(x)=\sum_{j=1}^M \beta_j b_j(x)$ in the convex hull of $\cB$,  whose excess risk, when compared to the best model in the convex hull of $\cB$ satisfies the inequality
\begin{equation*}
  \bbE L(y\la\beta ,b(x)\ra)\leq \min_{\theta\in \Delta_M}  \bbE L(y\la\theta,b(x)\ra)+\delta_{T,M},
\end{equation*}
where $\delta_{T,M}>0$ is a small remainder term that goes to 0 as $T\rightarrow \infty$, and the expectation is w.r.t. all the random variables involved.  In order to construct such a $\beta$ vector, we assume that we have access to an unlabeled stream of examples $x_1,x_2,\ldots$, drawn i.i.d. from the underlying distribution $P$ defined on $\cX$. This problem matches the one posed by Nemirovski, except that instead of being given labeled data, we now have access to an unlabeled data stream, and are allowed to query for the labels which we believe are most informative.

~\citet{juditsky2005recursive} studied the original version of the convex aggregation problem posed by Nemirovski, which involved learning the best convex aggregation given that one has access to a stream of labeled examples sampled i.i.d. from the underlying distribution. They introduced an online, stochastic mirror descent algorithm for the problem of learning the best convex aggregation of models. We shall call their method SMD-PMA~\footnote{SMD-PMA stands for \underline{S}tochastic \underline{M}irror \underline{D}escent based \underline{P}assive \underline{M}odel \underline{A}ggregation }. They showed that by making one pass of the stochastic mirror descent algorithm, and by averaging the iterates obtained after each step of the algorithm, the resulting convex aggregate has excess risk of $O\left(\sqrt{\frac{\log(M)}{T}}\right)$, where $M$  is the number of models being aggregated, and $T$ is the number of samples seen in the stream.  Essentially, SMD-PMA  is a slight modification of the stochastic mirror descent algorithm, applied to the stochastic optimization problem 
\begin{equation*}
\min_{\theta\in \Delta_M}\bbE L(y\la\theta,b(x)\ra),
\end{equation*} 
where $\Delta_M$ is the standard $M-1$ dimensional probability simplex. Since the constraint set is a simplex, the entropy regularizer was used in SMD-PMA. The stochastic mirror descent procedure is followed by an averaging step, that allows the authors to obtain excess risk bounds.

\textbf{Contributions.}  In this paper, we are interested in learning convex aggregation of models, with the help of actively labeled data. We consider a streaming setting, where we are given unlabeled points $x_1,x_2,\ldots$ and an oracle $\cO$. The oracle $\cO$, when provided as an input $x$ in $\cP$, returns a label $y\in \{-1,+1\}\sim \bbP[Y|X=x]$. We present an algorithm which is essentially a one-pass, stochastic mirror-descent based active learning algorithm, called SMD-AMA~\footnote{\underline{S}tochastic \underline{M}irror \underline{D}escent based \underline{A}ctive \underline{M}odel \underline{A}ggregation}, which solves the stochastic optimization problem $\min_{\theta\in\Delta_M} \bbE_{x,y}[L(y\la\theta,b(x)\ra)]$. Since we are dealing with simplex constraints, we use the entropy function as the regularization function in our stochastic mirror descent algorithm. In round $t$ of SMD-AMA, we query for the label of point $x_t$, with probability $p_t$. This allows us to construct an unbiased stochastic subgradient of the objective function at the current iterate. If the length of the stream is $T$ points, then SMD-AMA returns the hypothesis $\bhatT(x)\defeq \la \thetahatT,b(x)\ra$, $\thetahatT\in \Delta_M$. We show that the excess risk of the hypothesis $\bhatT$, w.r.t. the best convex aggregate of models, scales with the number of models as $\sqrt{\log(M)}$, and decays with the number of points, $T$, as $\frac{1}{\sqrt{T^{1-\mu}}}$, where $\mu\in [0,1)$ is an algorithmic parameter. The mild dependence on the number of models, $M$, allows us to use a large number of models, which is desirable when we are learning convex aggregation of models. 

The paper is organized as follows. In Section~\ref{sec:agg_alg}, after briefly reviewing the  stochastic mirror descent algorithm, we introduce our algorithm SMD-AMA. In Section~\ref{sec:agg_excess_risk} we present an excess risk bound for the hypothesis returned by our active learning algorithms. Section~\ref{sec:agg_related} reviews related work, and Section~\ref{sec:agg_exp} compares our proposed algorithm to a passive learning algorithm, and a previously proposed ensemble based active learning algorithm.
\section{A brief introduction to mirror descent} Mirror descent (MD)~\citep{beck2003mirror} is a first order optimization algorithm used to solve convex optimization problems of the form $\min_{x\in \cC} f(x)$. MD assumes the presence of a (sub)gradient oracle, which when provided a point $x\in\cC$, returns the sub-gradient $\nabla f(x)$. In each iteration of the mirror descent algorithm, given the current iterate $x_t$, MD solves a minimization problem of the form $\min_{x\in \cC} \la \nabla f(x_t),x\ra+D_{\cR}(x,x_t)$, where $\nabla f(x_t)$ is the sub-gradient of $f$ at $x_t$ and $D_{\cR}(\cdot,\cdot)$ is the Bregman divergence corresponding to a strictly convex function $\cR$, which we shall call as the regularization function. The power of the MD algorithm stems from the fact that by an appropriate choice of the distance generating function, one could adapt to the geometry of the constraint set $\cC$. The classical mirror descent algorithm has also been extended to stochastic optimization problems of the form $\min_{x\in\cC}\{f(x)=\bbE_{\omega} F(x;\omega)\}$, to yield the stochastic mirror descent (SMD) algorithm~\citep{nemirovski2009robust,lan2011validation}. In SMD, we assume that an oracle, provides us with an unbiased estimate of the gradient to the stochastic objective.
\section{Algorithm Design}
\label{sec:agg_alg}
Standard analysis of stochastic mirror descent algorithm assumes that we have access to a stochastic (sub)gradient oracle which provides an unbiased estimate of the gradient of the objective function at any point in the domain. The stochastic optimization problem that we are trying to solve is 
\begin{equation}
  \label{eqn:stochastic_opt}
  \min_{\theta\in \Delta_M} \{f(\theta)\defeq \bbE L(y\la\theta,b(x)\ra)\}.
\end{equation}
A naive application of the stochastic mirror descent method would require, in each iteration, to obtain a stochastic subgradient of $\bbE L(y\la\theta,b(x)\ra)$. In iteration $t$, if the current iterate is $\theta_{t-1}$, then a stochastic subgradient of $f(\theta)$ at $\theta=\theta_{t-1}$ is given by 
\begin{equation}
  \label{eqn:stocsubgrad}
  \nabla f(\theta_{t-1})= L'(y_t\la\theta_{t-1},b(x_t)\ra)y_tb(x_t),
\end{equation} 
where $L'(\cdot)$ is the subderivative of $L$ at the given argument. If in round $t$, we decided to query for the label of the point $x_t$, then one can calculate the stochastic subgradient using Equation~\ref{eqn:stocsubgrad}. However, if we decided not to query for the label of $x_t$, then the stochastic subgradient, which  depends on the unknown label $y_t$ of $x_t$, cannot be calculated. While one could, in such a case, consider the stochastic subgradient to be the zero vector, this is no longer an unbiased estimate of the subgradient. This is problematic, as the classical analysis of stochastic mirror descent, assumes that one has access to unbiased estimates of the subgradient of the objective function. In order to counter this problem we use the idea of importance weighting. 

\textbf{Importance weighted subgradient estimates:} In order to use importance weights, in round $t$ upon seeing a point $x_t$, we query $x_t$ with probability $p_t$. Suppose $Q_t$ is a $\{0,1\}$ random variable, which takes the value 1, if $x_t$ was queried, and takes the value 0, if $x_t$ was not queried. Let $Z_t\defeq y_t Q_t$. We shall also make the following independence assumption.
\begin{assumption}
  \label{ass:indep}
  $ p_t\indep y_t|x_t,x_{1:t-1},Z_{1:t-1}$,
\end{assumption}
where $Z_{1:t-1}$ is the collection of random variables $Z_1,\ldots,Z_{t-1}$.
Consider the following importance-weighted stochastic subgradient $g_t\defeq\frac{Q_t}{p_t}L'(y_t\la \theta_{t-1},b(x_t)\ra)y_tb(x_t)$. It is easy to see that, under Assumption~\ref{ass:indep}, given $\theta_{t-1}$, $g_t$ is an unbiased estimator of $\nabf(\theta_{t-1})$. Hence, the above importance weighted stochastic subgradient  allows us to get unbiased estimates of the gradient of the objective function in Equation~\ref{eqn:stochastic_opt}. Importance weighting was first introduced for stream based active learning problems by~\citet{beygelzimer2009importance}. The use of importance weights helps us construct unbiased estimators of the loss of any hypothesis in a hypothesis class.  This is particularly favourable, when the hypothesis class used for obtaining actively labeled dataset, goes out of favour.

Before we dive into the details of SMD-AMA we shall need some notation, and terminology, that is relevant for the explanation of the our algorithm. Most of this exposition is standard and is taken from~\cite{juditsky2005recursive}. Let $E=\ell_1^M$, be the space of $\bbR^M$ equipped with $\ell_1$ norm. Let $E^{*}=\ell_{\infty}^M$ be the corresponding dual space, equipped with the $\ell_{\infty}$ norm.  
\begin{defn}
  Let $\Delta_M\subset E$, be the simplex, and let $V:\Delta_M\rightarrow \bbR$ be a convex function. For a given parameter $\beta$, the $\beta$ Fenchel dual of $V$ is  the convex function $V^{*}_\beta:E^*\rightarrow\bbR$, defined as
\begin{equation*}
  V^{*}_{\beta}(\xi)=\sup_{\theta\in\Delta_M}[-\la\xi,\theta\ra-\beta V(\theta)].
\end{equation*}
\end{defn}
Finally, as mentioned in Section~\ref{sec:agg_intro}, to define the SMD algorithm, we need a regularization function. Since our constraint set is a simplex, we shall use the entropy function defined as 
\begin{equation*}
  \cR(\theta)=
  \begin{cases}
    -\sum_{j=1}^M \theta_j\log(\theta_j)~\text{if~} \theta\in \Delta_M\\
    \infty~\text{otherwise}
  \end{cases}
\end{equation*}
as our regularization function.

\textbf{Design of SMD-AMA.} 
We now have all the ingredients of our algorithm in place. The algorithm proceeds in a streaming fashion, looking at one unlabeled data point at a time. Step 4 of SMD-AMA, calculates the probability of the point being labeled +1 by the current convex aggregate, and this calculation is used in Step 5 to calculate the probability of querying the label of $x_t$. Notice that the probability of querying a point, in round $t$, is always at least $\epsilon_t>0$. Value of $\epsilon_t$ is set in Step 3. Step 7 calculates the importance weighted gradient, which is used in Step 9 to calculate the new iterate $\theta_t$. By straightforward calculus, one can show that Step 9, leads to the following iteration
\begin{equation*}
  \theta_{t,j}\propto\exp(-\xi_{t,j}/\beta),
\end{equation*}
where $\xi_{t,j}$ is the $j^{\text{th}}$ component of the vector $\xi_t$.
In Step 4, we use properties of the loss function in order to estimate $p_t^{+}$. It is well known that standard loss functions such as exponential, squared loss, which are used in classification, are also proper losses for probability estimation~\citep{zhang2004statistical,buja2005loss,reid2009surrogate}. Hence, given the loss function, via standard formulae it is easy to estimate the conditional probability $\bbP[Y_t=1|X_t=1]$. For instance, if one were to use the squared loss $L(yz)=(1-yz)^2$, then our estimate for $p_t^{+}$ is given by the formula $\max\left(0,\min\left(\frac{1+z}{2},1\right)\right)$. In our case, the value of $z$, in round $t$, is given by $\la\theta_{t-1},b(x_t)\ra$, and hence, $p_t^{+}=\max\left(0,\min\left(\frac{1+\la\theta_{t-1},b(x_t)\ra}{2},1\right)\right)$.
\begin{algorithm}[h]
  \caption{\label{alg:smd_ama}{\color{blue}SMD-AMA} (Input: A margin based loss function $L$, Labeling Oracle $\cO$, Parameters $\mu\in[0,1)$, $\beta_0>0$)}
  \begin{algorithmic}
    \STATE 1. Initialize $\theta_0=[\frac{1}{M},\ldots,\frac{1}{M}]^T,\xi_0=[0,\ldots,0], t=1$
    \FOR{t=1,\ldots}
    \STATE 2. Receive $x_t$.
    \STATE 3. Set $\epsilon_t=t^{-\mu}$
    \STATE 4. Estimate $p_t^{+}=\bbP[Y_t=1|X=x_t,\theta_{t-1}]$
    \STATE 5. Calculate $p_t=4p_t^{+}(1-p_t^{+})(1-\epsilon_t)+\epsilon_t$.
    \STATE 6. Query the label of $x_t$ with probability $p_t$.
    \STATE 7. Set $g_t=\frac{Q_t}{p_t}L'(y_t\la\theta_{t-1},b(x_t)\ra)y_tb(x_t)$.
    \STATE 8. Set $\xi_t\leftarrow \xi_{t-1}+g_t$.
    \STATE 9. Update $\beta_t=\beta_0(t+1)^{\frac{1+\mu}{2}}$.
    \STATE 10. Calculate $\theta_t=-\nabla \cR^{*}_{\beta_t}(\xi_t)$.
    \STATE 11 $t\leftarrow t+1$.
    \ENDFOR
    \STATE 12. Return $\thetahatT=\frac{\sum_{t=1}^T \theta_t}{T}$
  \end{algorithmic}
\end{algorithm} 

A technical point is that, Step 9 in Algorithm~\ref{alg:smd_ama}, is performed irrespective of whether the point received in round $t$ is queried or not. This allows the evolution of $\beta_t$ to be deterministic, and hence allows us to use a certain lemma of Juditsky et al~\citeyearpar{juditsky2005recursive}, that is crucial to our excess risk analysis . As a result of this step, our current convex aggregate $\theta_t$ will be different from $\theta_{t-1}$, even if the point $x_t$ was not queried. 
\section{Excess Risk analysis}
\label{sec:agg_excess_risk}
\begin{thm}
\label{thm:excess_risk}
Let $\cB=\{b_1,\ldots,b_M\}$ be a collection of basis models, where for each $x\in\cX$, $b_j(x)\in\{-1,+1\}$. For any $x$ in $\cX$, let $b(x)=[b_1(x),\ldots,b_M(x)]^T$. Let, $R(\theta)\defeq \bbE L(y\la\theta,b(x)\ra)$. Then, for any $T\geq 1$, and any $\theta\in\Delta_M$, when SMD-AMA is run with the parameter $\beta_0=\sqrt{\frac{L_{\phi}^2}{2\log(M)\sqrt{2^{\mu+1}}(1+\mu)}}$, where $\mu\in [0,1)$, then the convex aggregation, $\thetahatT$, returned by the SMD-AMA algorithm after $T$ rounds, satisfies the following excess risk inequality
\begin{equation*}
  \bbE R(\thetahatT)\leq  R(\theta)+2\sqrt{T^{\mu-1}}\left(\sqrt{\frac{L_{\phi}^2\sqrt{2^{\mu-1}}\log(M)}{1+\mu}}\right).
\end{equation*}
\end{thm}
\subsection{Proof of Theorem~\ref{thm:excess_risk}} 
Let $R(\theta)\defeq\bbE L(y\la\theta,b(x)\ra)$.
By definition, 
\begin{equation*}
  \thetahatT=\frac{\sum_{t=1}^T\thetatminus}{T}.
\end{equation*}
Since $L$ is a convex function, hence by Jensen's inequality
\begin{align}
  \label{eqn:jensen}
  \bbE R(\thetahatT)-\bbE R(\theta)&\leq \frac{\sum_{t=1}^T \bbE R(\thetatminus)-\bbE R(\theta)}{T}
\end{align}
 Since $R(\theta)$ is a convex function of $\theta$, hence we can use the subgradient to build an under-approximation to get
\begin{equation}
  \label{eqn:subgrad}
  R(\thetatminus)-R(\theta)\leq \la\nabR(\thetatminus),-\theta+\thetatminus\ra.
\end{equation} 
Putting together Equations~\ref{eqn:jensen},~\ref{eqn:subgrad} we then get 
\begin{align*}
  \bbE R(\thetahatT)-\bbE R(\theta)&\leq \frac{\sum_{t=1}^T\bbE\la\nabR(\thetatminus),\thetatminus-\theta\ra}{T}\\
  &\leq \frac{1}{T}\Bigl[\beta_T\log(M)+\bbE \sum_{t=1}^T \frac{||g_t||_{\infty}^2}{2\beta_{t-1}} -\bbE\sum_{t=1}^T \la\theta_{t-1}-\theta,
  g_t-\nabla R(\theta_{t-1})\ra\Bigr].
\end{align*}
In the above set of inequalities, to obtain the second inequality from the first we used Lemma 1 (presented in the supplementary material). Since our data stream is drawn i.i.d. from the input distribution, hence $\bbE g_t-\nabla R(\thetatminus)=0$. This gets us
\begin{align*}
  \bbE R(\thetahatT)-\bbE R(\theta)&\leq \frac{\bbE\la\nabR(\thetatminus),\thetatminus-\theta\ra}{T}\\
  &\leqa \frac{1}{T}\left[\beta_T\log(M)+\sum_{t=1}^T \bbE\frac{L_{\phi}^2}{2p_t\beta_{t-1}}\right]\\
  &\leqb\frac{\beta_T \log(M)}{T}+\frac{L_{\phi}^2}{2T\beta_0}\sum_{t=1}^T \sqrt{t^{\mu-1}}\\
&\leq 2\sqrt{T^{\mu-1}}\sqrt{\frac{L_{\phi}^2\sqrt{2^{\mu-1}}\log(M)}{1+\mu}}.
\end{align*}
Inequality (a) was obtained by using the fact that $\bbE[\frac{Q_t}{p_t}]=1$, and by approximating $|L'(y_i\la \theta,b(x_i)\ra)|\leq L_{\phi}$. Inequality (b) was obtained by substituting for $\beta_{t-1}$ the expression $\beta_0 \sqrt{t^{1+\mu}}$, and upper bounding $\frac{1}{p_t}$ with $\frac{1}{\epsilon_t}$. Replacing $\beta_T$ in terms of $\beta_0,T$, and using the expression for $\beta_0$ as mentioned in the statement of the theorem, and over approximating $T+1$ by $2T$ we get the desired result. 
\section{Related Work}
\label{sec:agg_related}
~\citet{madani2004active} considered the problem of active model aggregation, where given many models one has to choose a single best model from the collection. They model this problem as the coins problem, where a player is provided with a certain number of flips, and is allowed to flip coins until the budget runs out, after which the player has to report the coin which has the highest probability of turning up heads. A reinforcement learning approach, to the same problem, was taken by~\citet{kapoor2005reinforcement}. In this paper, we are dealing with a more complicated problem of choosing the best model from the convex hull of a given set of models.~\citet{mamitsuka1998query} combine the ideas of Query-by-committee, and boosting to come up with an active learning algorithm, where the current weighted majority is used to decide which point to query next. In order to obtain a weighted-majority of hypothesis, the authors suggest using ensemble techniques such as boosting and bagging.~\citet{trapeznikovactive} introduced the ActBoost algorithm, which is an active learning algorithm in the boosting framework.  Particularly, ActBoost works under the weak learning assumption of~\citet{freund1996experiments}, which assumes that there is a hypothesis with zero error rate, in the convex hull of the base classifiers. Under this assumption, the authors suggest a version space based algorithm, that maintains all the possible convex combinations of the base hypothesis that are consistent with the current data, and queries the labels of the points, on which two hypothesis in the current convex hull, disagree. By design, the ActBoost algorithm is very brittle. In contrast, we do not make any weak learning assumptions, and hence avoid the problems that ActBoost might face when weak learning assumption is not satisfied. Active learning algorithms, in the boosting framework, have also been suggested by~\citet{iyengar2000active}, but they do not admit any guarantees, and are somewhat heuristic.
\section{Experimental Results}
\label{sec:agg_exp}
We implemented SMD-AMA, along with SMD-PMA, and the QBB algorithm~\cite{mamitsuka1998query}. As mentioned before, QBB is an ensemble based  active learning algorithm, which builds a committee via the AdaBoost algorithm. QBB works in an iterative fashion. In round $t$, QBB runs the AdaBoost algorithm, on the currently labeled dataset $S_t$, with the collection of models $\cB$, to get a boosted model $h_t$. A boosted model is of the form $h_t(x)=\sum_{j=1}^M \alpha_{j,t} b_j(x)$, where $\alpha_{j,t}>0$, for all $j=1,\ldots,M$. To choose the next point to be queried, QBB generates a random sample of $R$ points from the current set of unqueried points. Suppose this random sample is $C_t$. To choose the next point to be queried, we look for that point in $C_t$, whose margin w.r.t. $h_t$ is the smallest. We then query for the label of this point. This process is repeated until some condition is satisfied (typically until a budget is exhausted).

 \subsection{Experimental Setup.}
 We used decision stumps along different dimensions, and with different thresholds, to form our set of basis models $\cB$. A decision stump is a weak classifier that is characterized by a dimension $j$, and a threshold $\theta$, and classifies a point $x\in \bbR^d$ as $\sgn(x_{j}-\theta)$, where $x_j$ is the $j^{\text{th}}$ dimension of $x$. For all our experiments, we used 80 decision stumps along each dimension~\footnote{Using more decision stumps yielded insignificant improvement in test error, but increased computational complexity by a large amount}. We make our set, $\cB$, symmetric by adding $-b$ to $\cB$, if $b\in\cB$. Unless otherwise mentioned, $\mu$ is set to 0.30 for all of our SMD-AMA experiments. We report results on some standard UCI datasets~\footnote{Our datasets are Abalone, Statlog, MNIST (3 Vs 5), Whitewine, Magic, Redwine.}. The loss function used for SMD-AMA, and SMD-PMA is the squared loss. 

\subsection{Comparison with Passive Learning} Our first set of experiments compare SMD-PMA to SMD-AMA. We run both SMD-AMA and SMD-PMA on our datasets, and  use the hypothesis outputted by these algorithms, at the end of each round, to classify on a test dataset. In Figure~\ref{fig:pma_ama_smd}, we plot the test error rate of both the algorithms with the number of points seen in the stream. Note that while SMD-PMA gets to see the label of each and every point, SMD-AMA gets to see only those labels which it queries. Alternatively, one could plot the test error of different algorithms w.r.t. the number of labels seen. However, as mentioned in Section~\ref{sec:agg_alg}, our hypothesis can change between two consecutive queries, and hence it is easier to plot the test error with the number of points seen, rather than number of labels seen. We used the squared loss function for training SMD-AMA, and SMD-PMA algorithms. Finally, since SMD-AMA is a randomized algorithm, we report results averaged over 10 iterations. 
\begin{figure*}[t]
  \centering
  \subfigure[Abalone]{\includegraphics[scale=0.20]{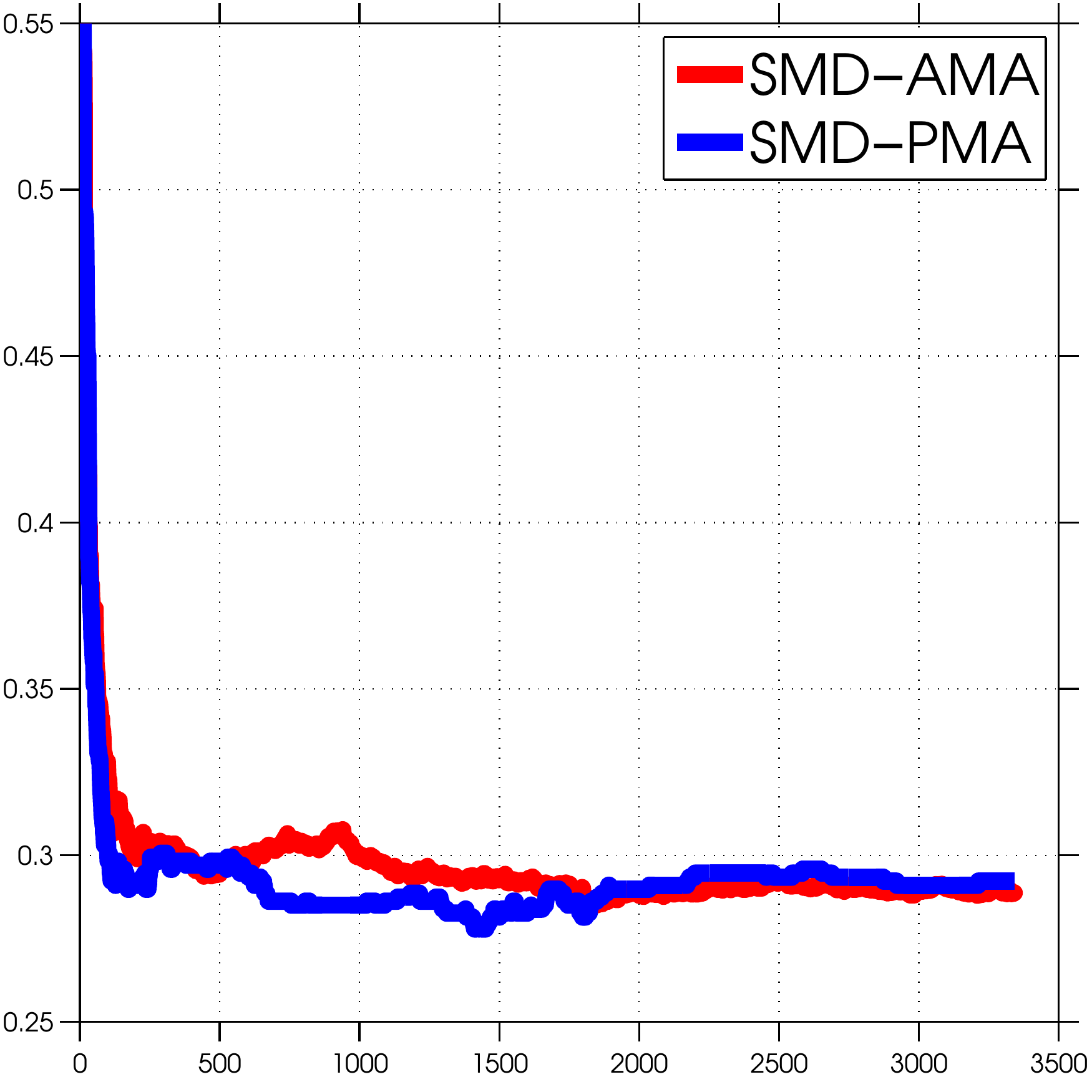}}
  \subfigure[Statlog]{\includegraphics[scale=0.20]{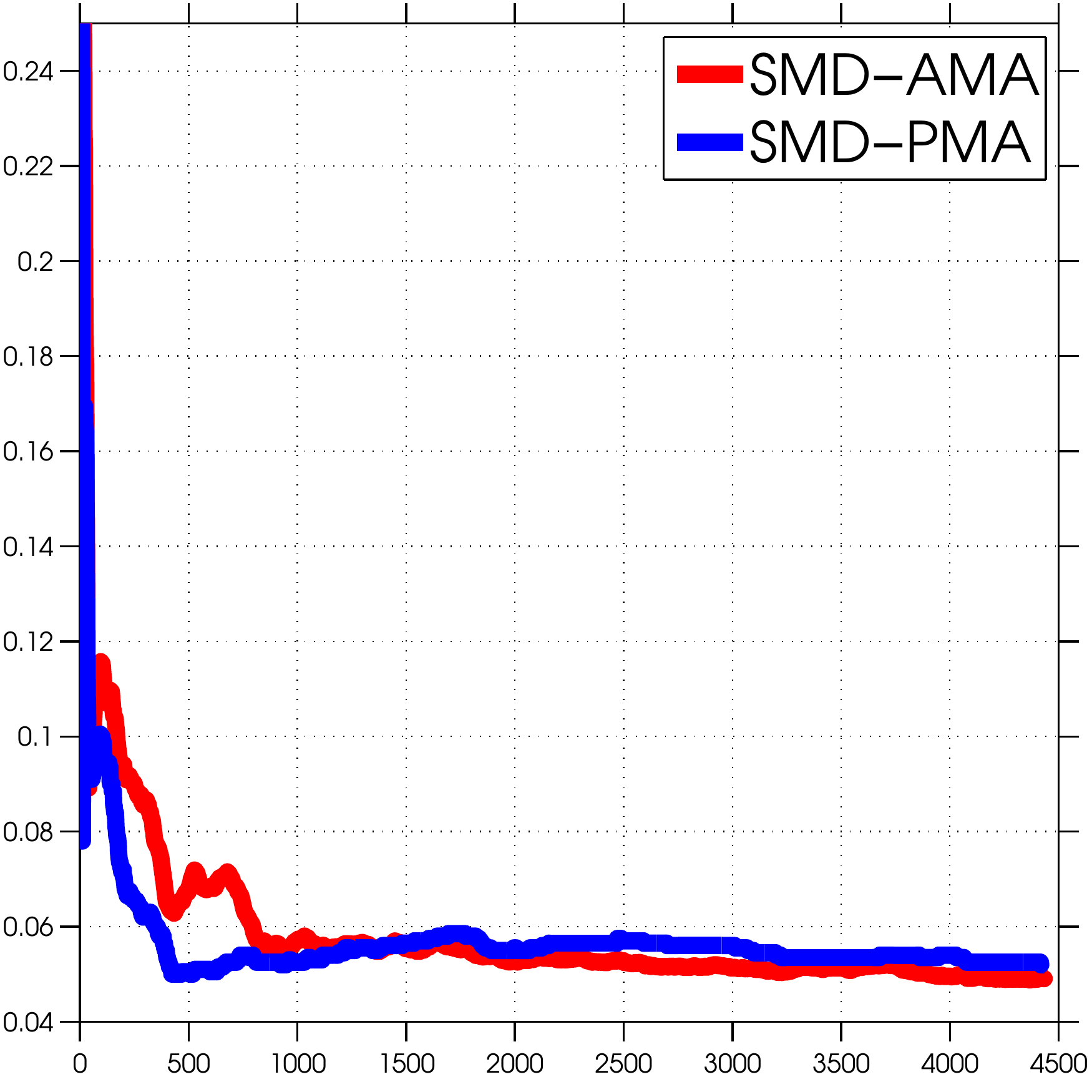}}
 \subfigure[MNIST (3 Vs 5)]{\includegraphics[scale=0.20]{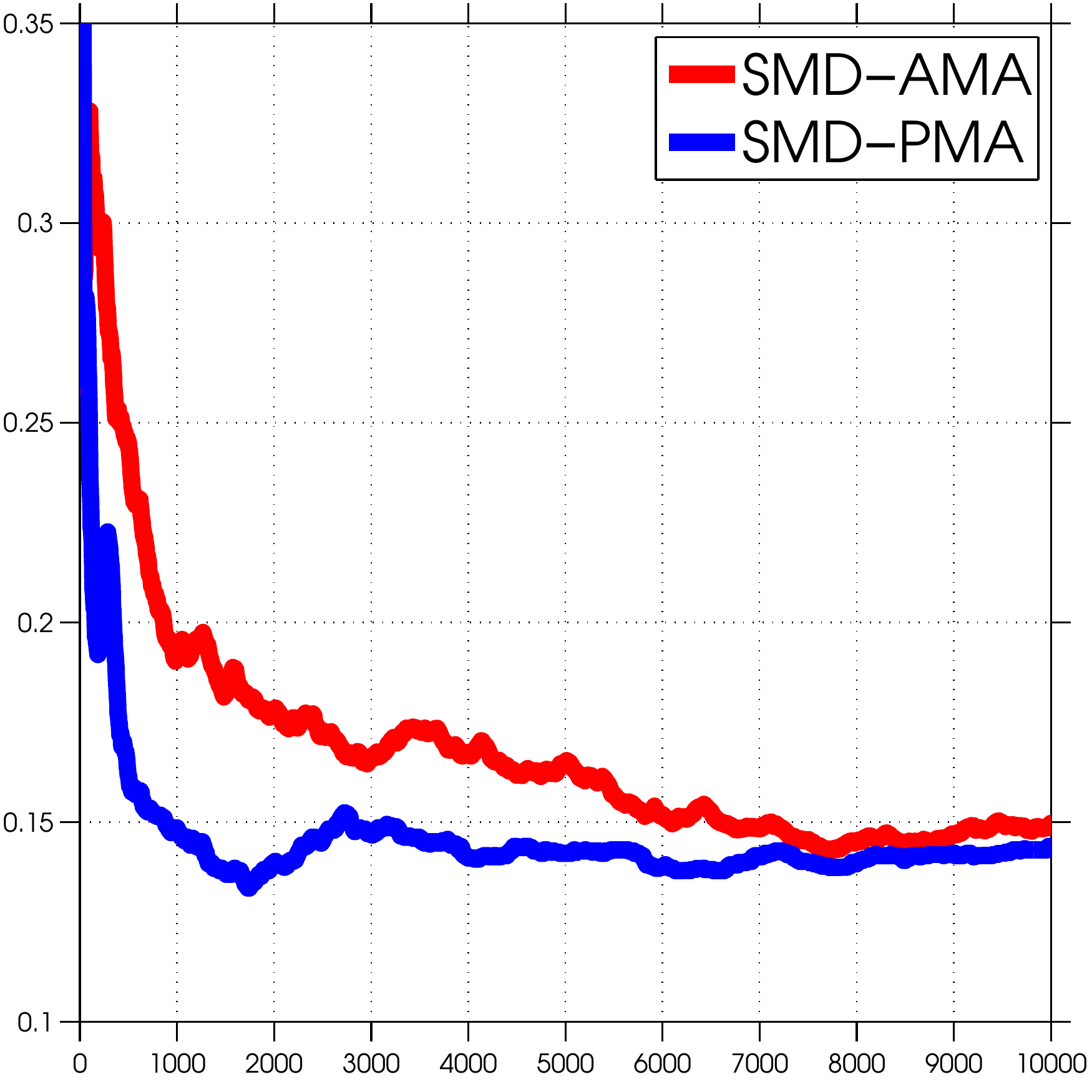}}
 \subfigure[Whitewine]{\includegraphics[scale=0.20]{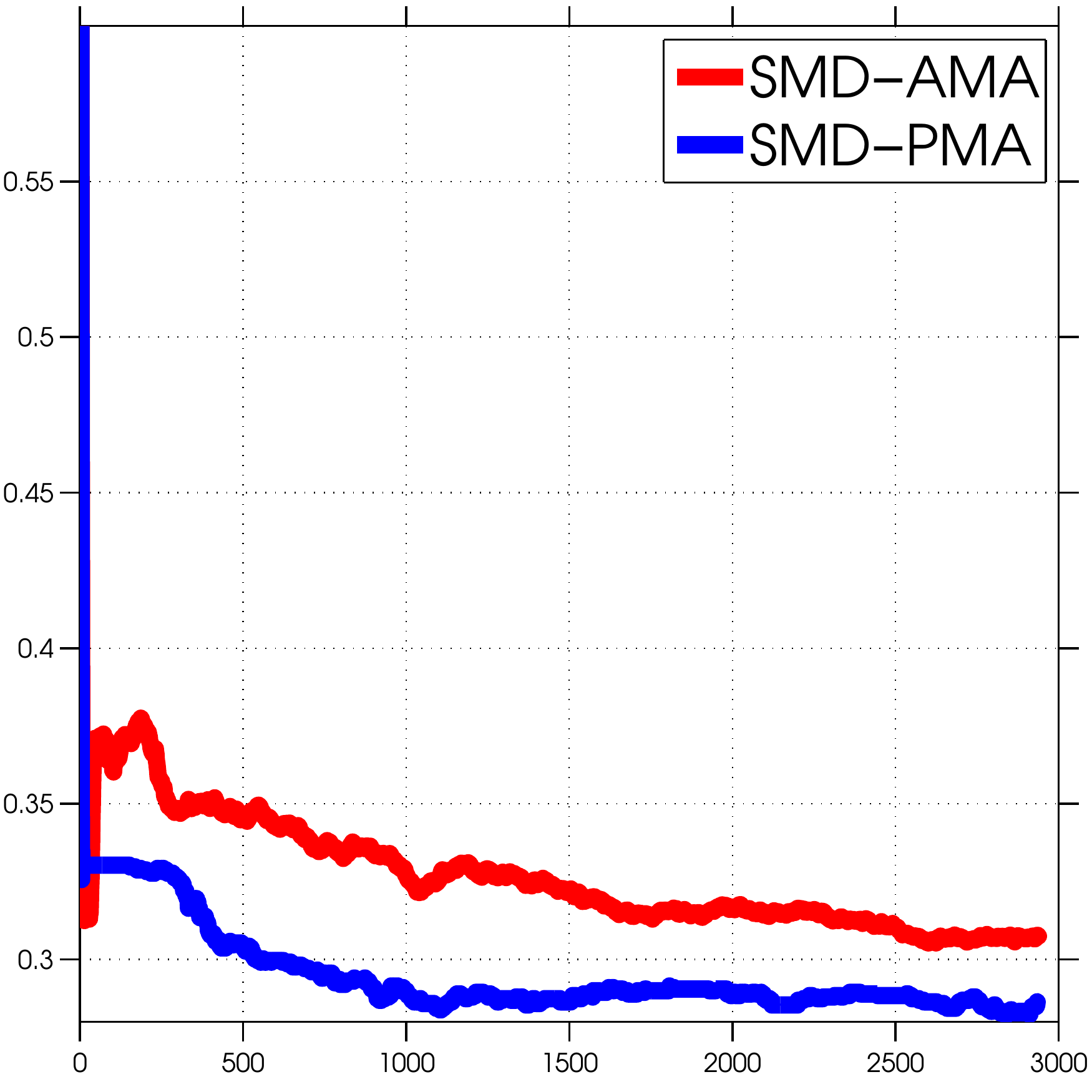}}
\subfigure[Magic]{\includegraphics[scale=0.20]{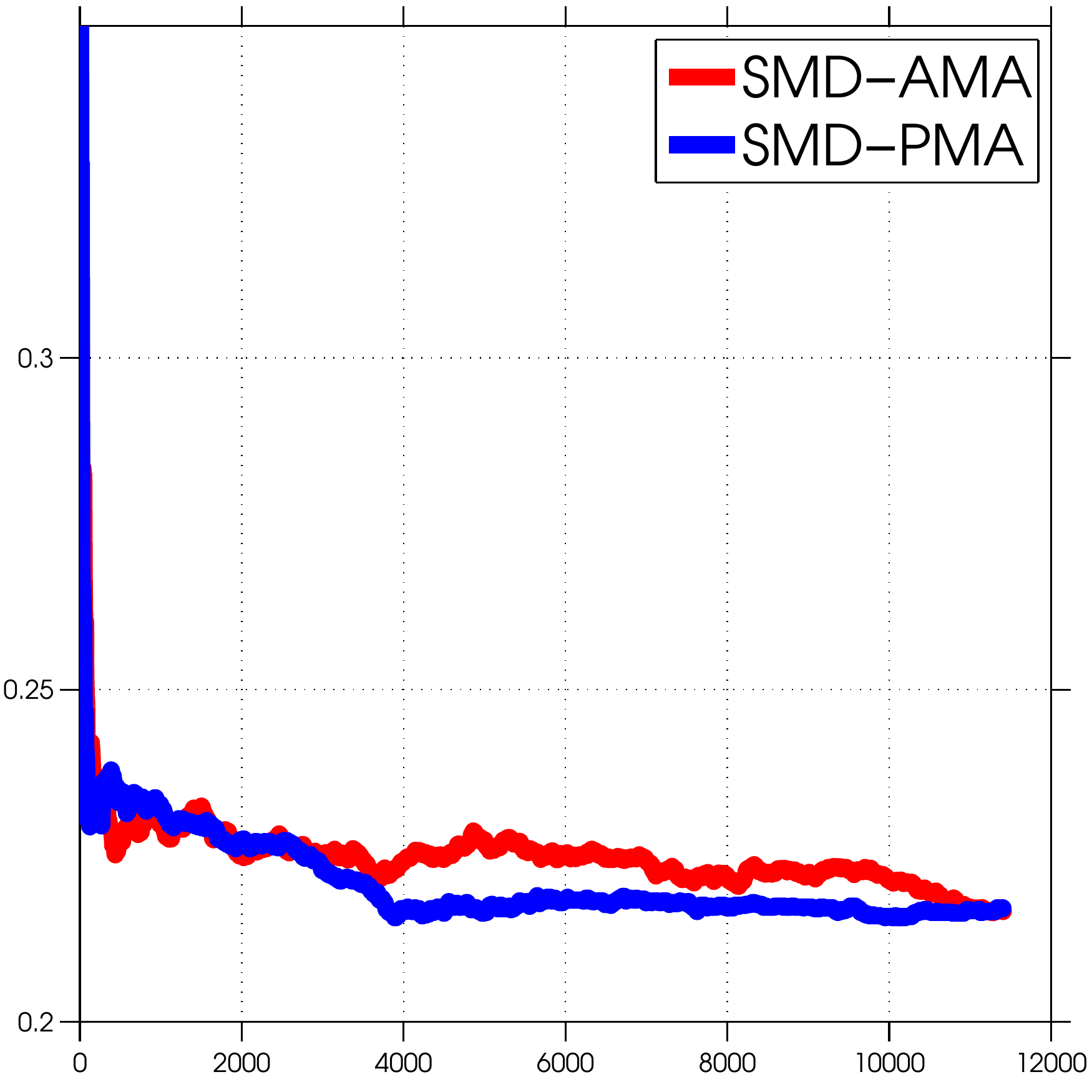}}
\subfigure[Redwine]{\includegraphics[scale=0.20]{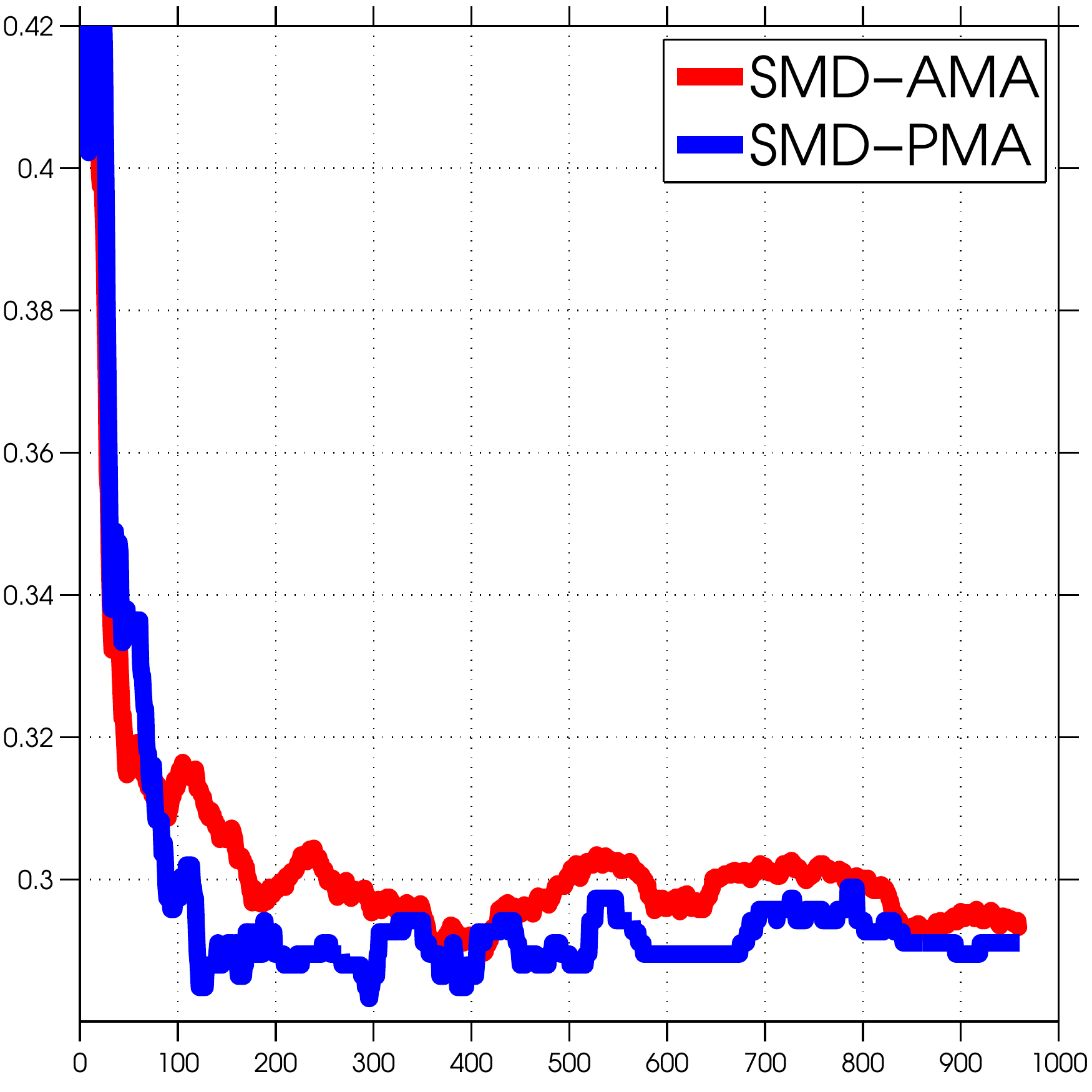}}
\caption{Comparison of the test error between SMD-AMA and SMD-PMA with the number of points seen.\label{fig:pma_ama_smd}}
\end{figure*}
\begin{table}[h]
\centering
\begin{tabular}{|l|l|l|l|l|}
\hline
Dataset&Active & Passive& Queries&Fraction\\
\hline
Abalone&0.2889&0.2922&440.2&0.1317\\
\hline
Statlog&0.0491&0.0520&2984&0.6728\\
\hline
MNIST&0.1496&0.1442&931.6&0.0932\\
\hline
Whitewine&0.3075&0.2864&406.8&0.1351\\
\hline
Magic&0.2166&0.2171&2450&0.2146\\
\hline
Redwine&0.2933&0.2911&540&0.5625\\
\hline
\end{tabular}
\caption{\label{tab:pma_ama_smd_err}Comparison of the test error between SMD-AMA (second column in the table) and SMD-PMA (third column in the table), and number of queries made by SMD-AMA on different datasets. The penultimate column reports the number of queries made by the algorithm, and the last column represents the fraction of the training dataset whose labels were queried. All results are for the hypothesis returned by the algorithms at the end of the stream.}
\end{table}
\begin{table}
\centering
\begin{tabular}{|l|l|l|l|l|l|l|}
\hline
Dataset&SMD-AMA & SMD-PMA\\
\hline
Abalone&0.7992&0.7611\\
\hline
Statlog&0.3242&0.3396\\
\hline
MNIST&0.5517&0.5381\\
\hline
Whitewine&0.7894&0.7361\\
\hline
Magic&0.6315&0.6305\\
\hline
Redwine&0.7455&0.7660\\
\hline
\end{tabular}
\caption{\label{tab:pma_ama_smd_loss}Comparison of the test loss between SMD-AMA and SMD-PMA for the hypothesis returned by the algorithms at the end of the stream.}
\end{table}
From Table~\ref{tab:pma_ama_smd_err}, it is clear that for all datasets but Whitewine, both SMD-AMA and SMD-PMA attain almost the same error rate, after finishing a single pass through the stream. Figure~\ref{fig:pma_ama_smd} shows how the test error changes for SMD-PMA and SMD-AMA with the number of points seen in the stream. While in the case of Abalone, and Statlog, SMD-AMA quickly catches up with SMD-PMA (and in fact slightly surpasses SMD-PMA), in the case of MNIST, the difference between SMD-AMA and SMD-PMA closes only after having seen about 80\% of the stream. In the case of Whitewine, SMD-PMA is uniformly better than our active learning algorithm, SMD-AMA. The difference in error rates, between SMD-PMA and SMD-AMA, at the end of the stream is about 1.32\%. In the case of Magic, Redwine datasets, the difference in performance of SMD-AMA and  SMD-PMA is negligible. 

The number of queries made in Abalone, MNIST, and Whitewine is less than 14\% of the length of the stream, which implies that we do as well as passive learning, for Abalone, and MNIST, at the expense of far fewer labels. In the case of Statlog, and Redwine datasets the number of queries made is comparatively larger, about 67.28\%, and 56.25\% of the size of the dataset respectively. On Magic the fraction of queries made is less than 25\% of the number of training points in the dataset.

In Table~\ref{tab:pma_ama_smd_loss}, we report the loss on the test data, of SMD-AMA, SMD-PMA at the end of the data stream. Figure~\ref{fig:loss_curve} reports the test loss of both the algorithms with the number of samples seen in the stream. Note that while test error is always between 0 and 1, the test loss can be larger than 1. In fact for the convex aggregation model that we consider in this chapter, and with the squared loss, the maximum loss can be as large as $\max_{z\in [-1,1]}(1-z)^2=4$. The purpose of these experiments was to examine how the difference between the test loss suffered by SMD-AMA, and SMD-PMA, changes with the number of points seen in the stream. In the case of Statlog, and MNIST the difference in losses is generally smaller than in the case of Abalone and Whitewine, and on Magic, Redwine datasets SMD-AMA generally has suffers slightly smaller loss than SMD-PMA. However, since the scale for test loss is larger than 1, these results seem to imply that both SMD-AMA and SMD-PMA have comparable rates of decay for the test loss, with the number of unlabeled examples seen. Encouraged by these results, we believe it is possible to derive sharper excess risk bounds for SMD-AMA.
\begin{figure*}
  \centering \subfigure[Statlog]{\includegraphics[scale=0.20]{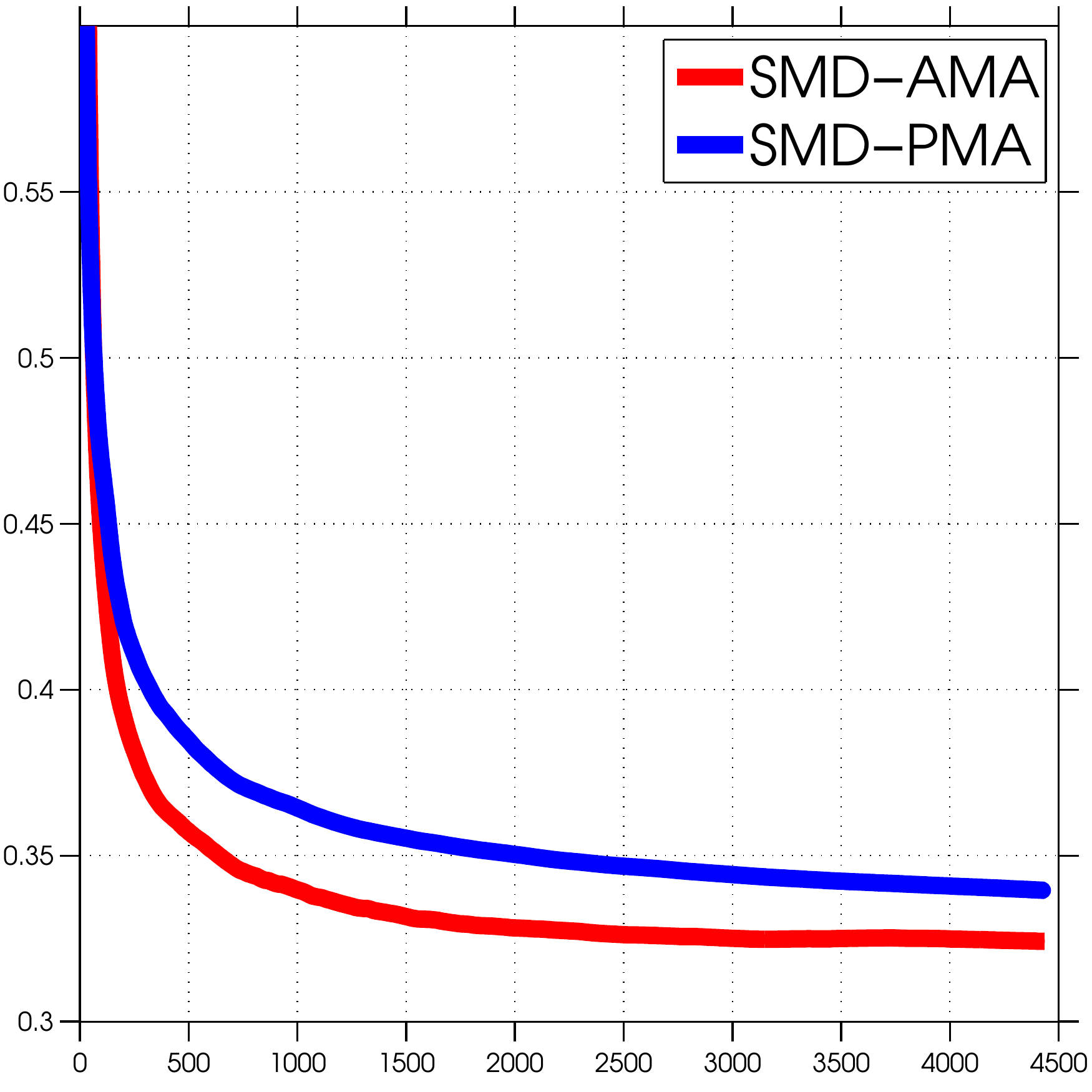}}
\subfigure[Magic]{\includegraphics[scale=0.20]{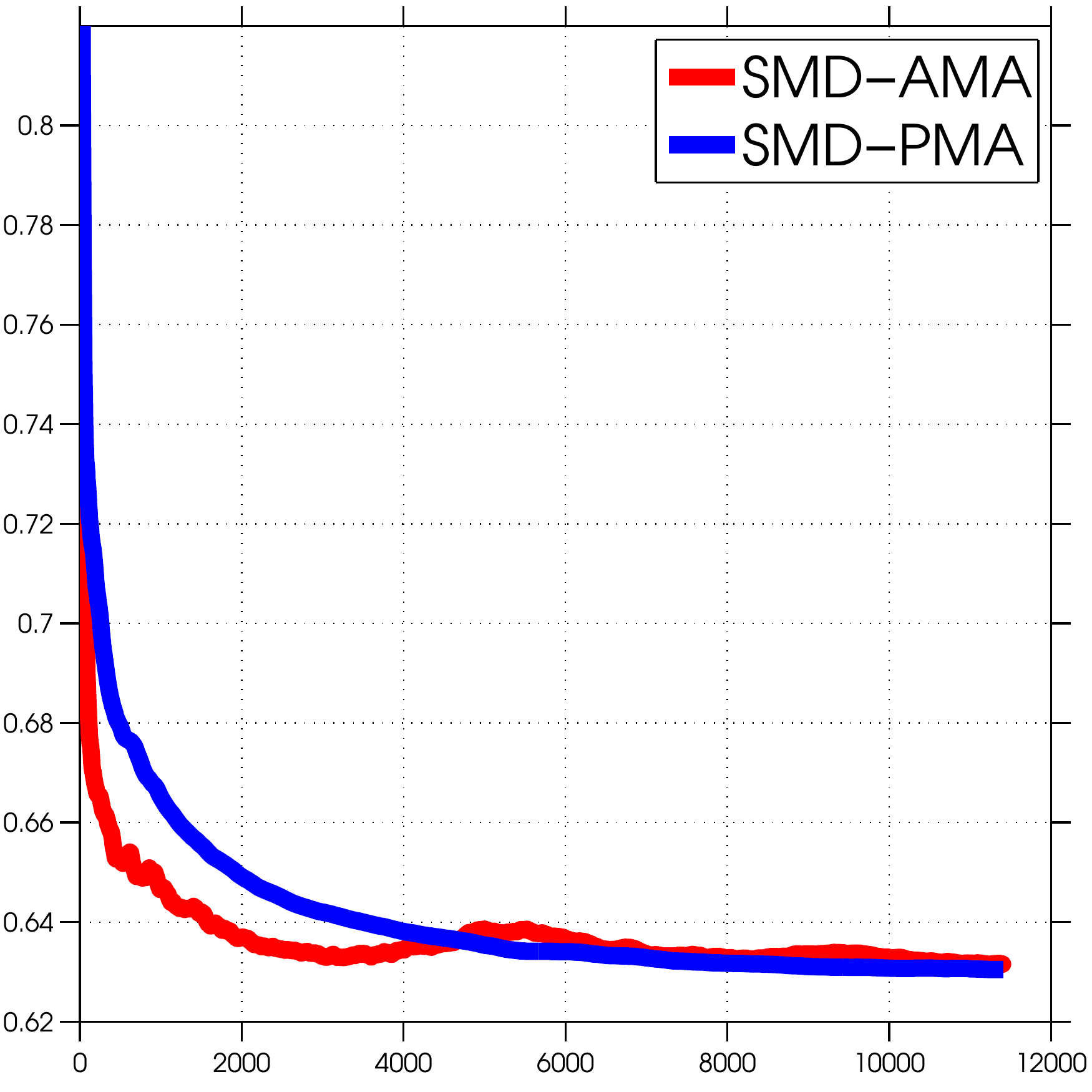}}
\subfigure[Redwine]{\includegraphics[scale=0.20]{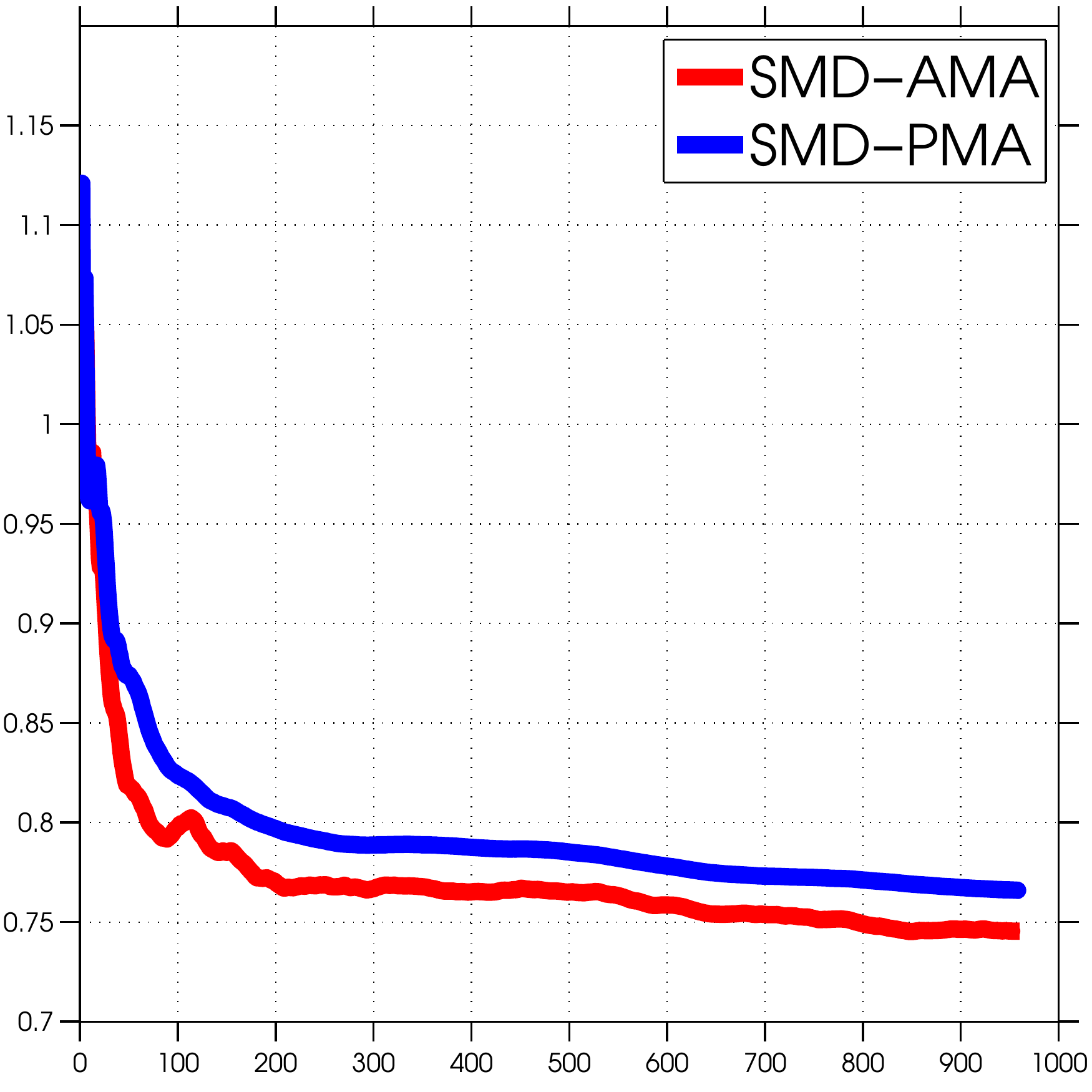}}
\subfigure[Abalone]{\includegraphics[scale=0.20]{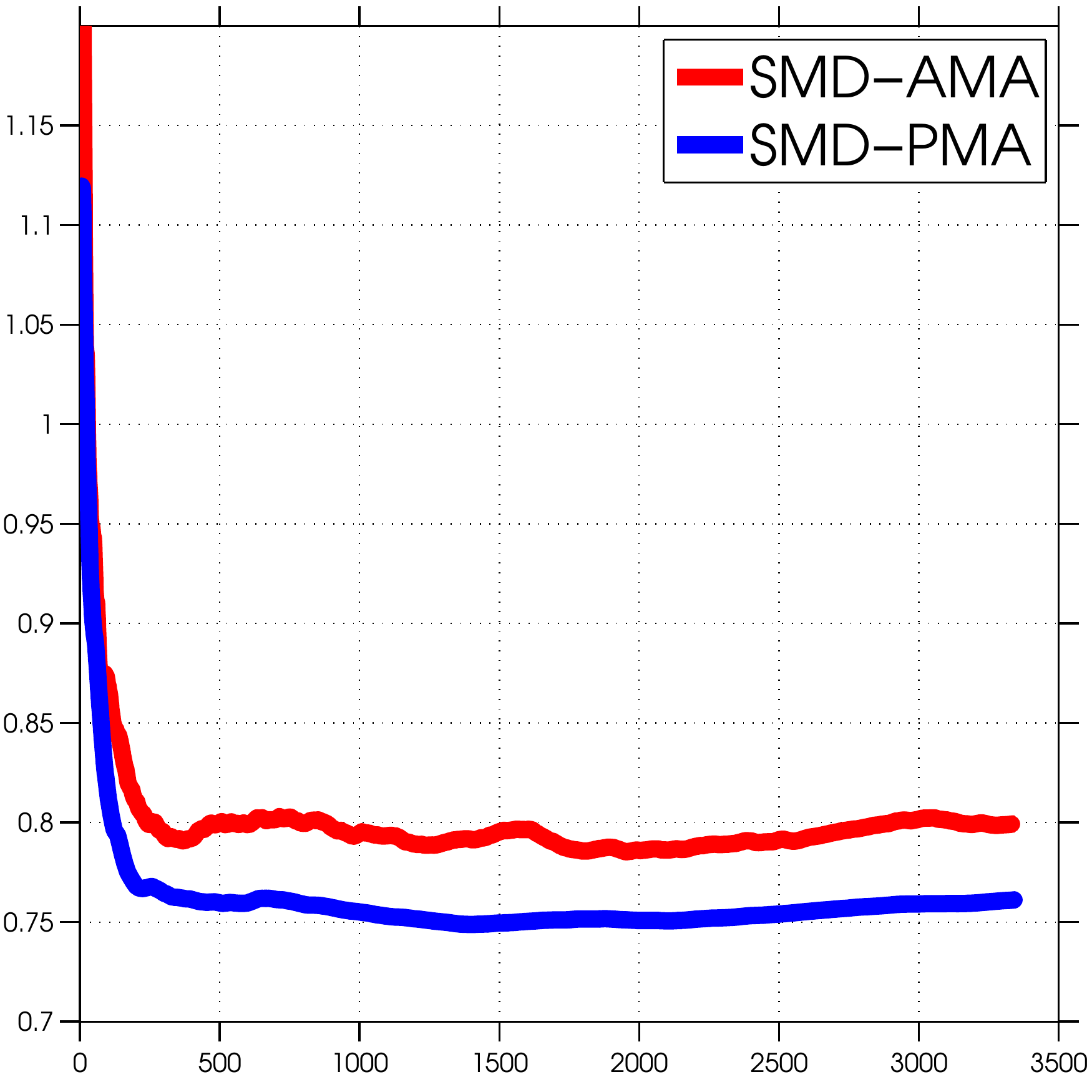}} 
\subfigure[MNIST]{\includegraphics[scale=0.20]{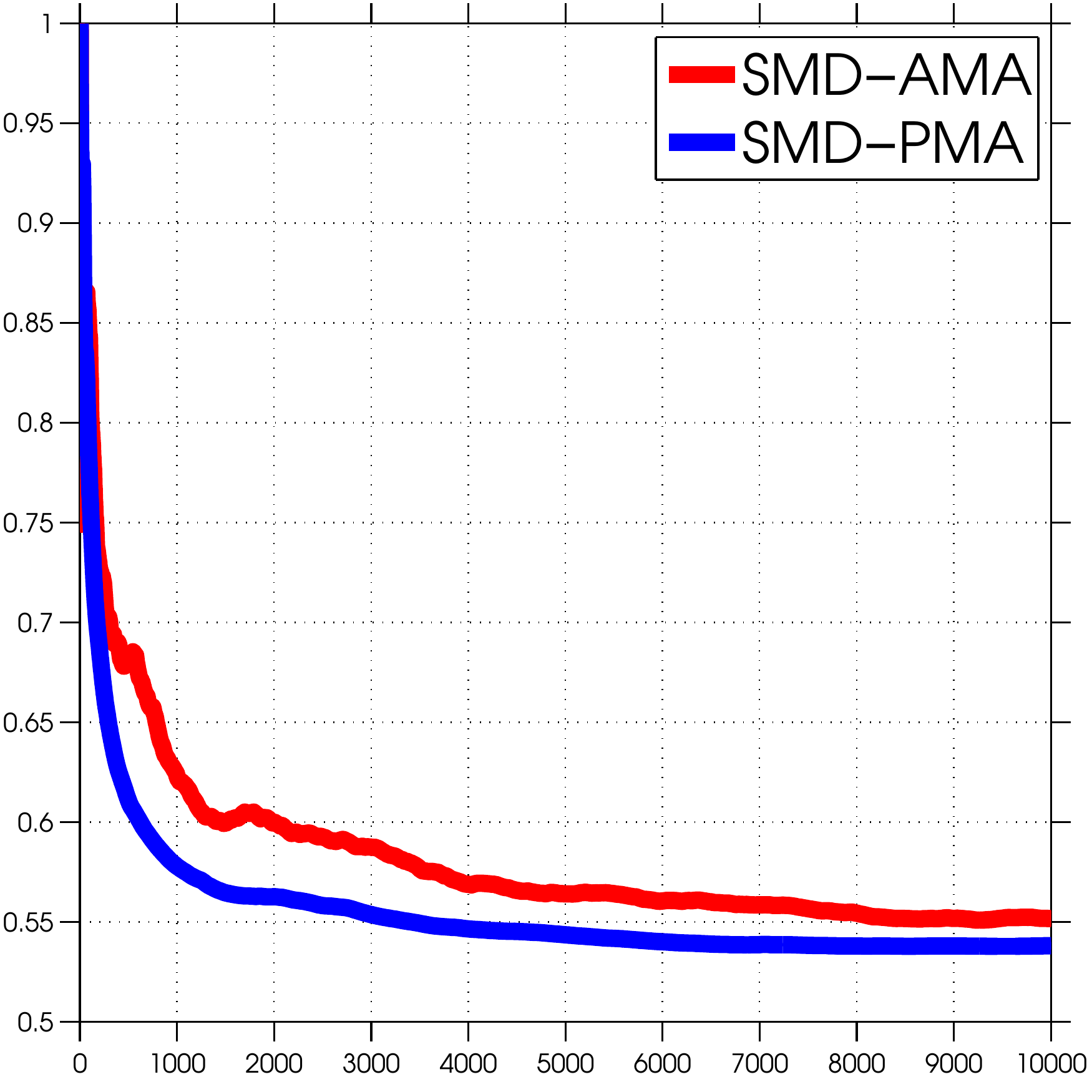}} \subfigure[Whitewine]{\includegraphics[scale=0.20]{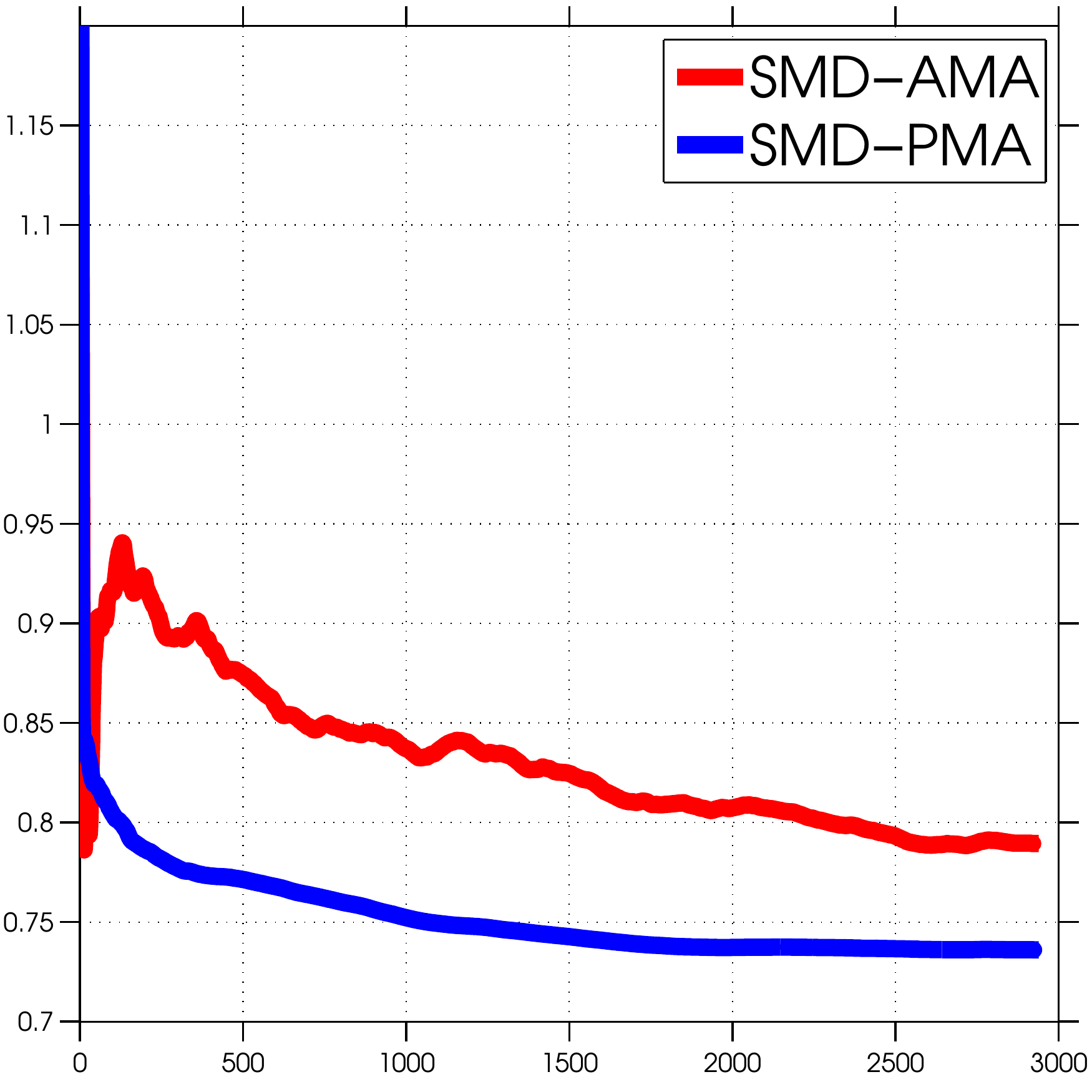}}
\caption{\label{fig:loss_curve}Comparison of the test loss of SMD-AMA and SMD-PMA with the number of points seen.}
\end{figure*} 
\begin{figure*}[hbpt]
  \centering \subfigure[Abalone]{\includegraphics[scale=0.15]{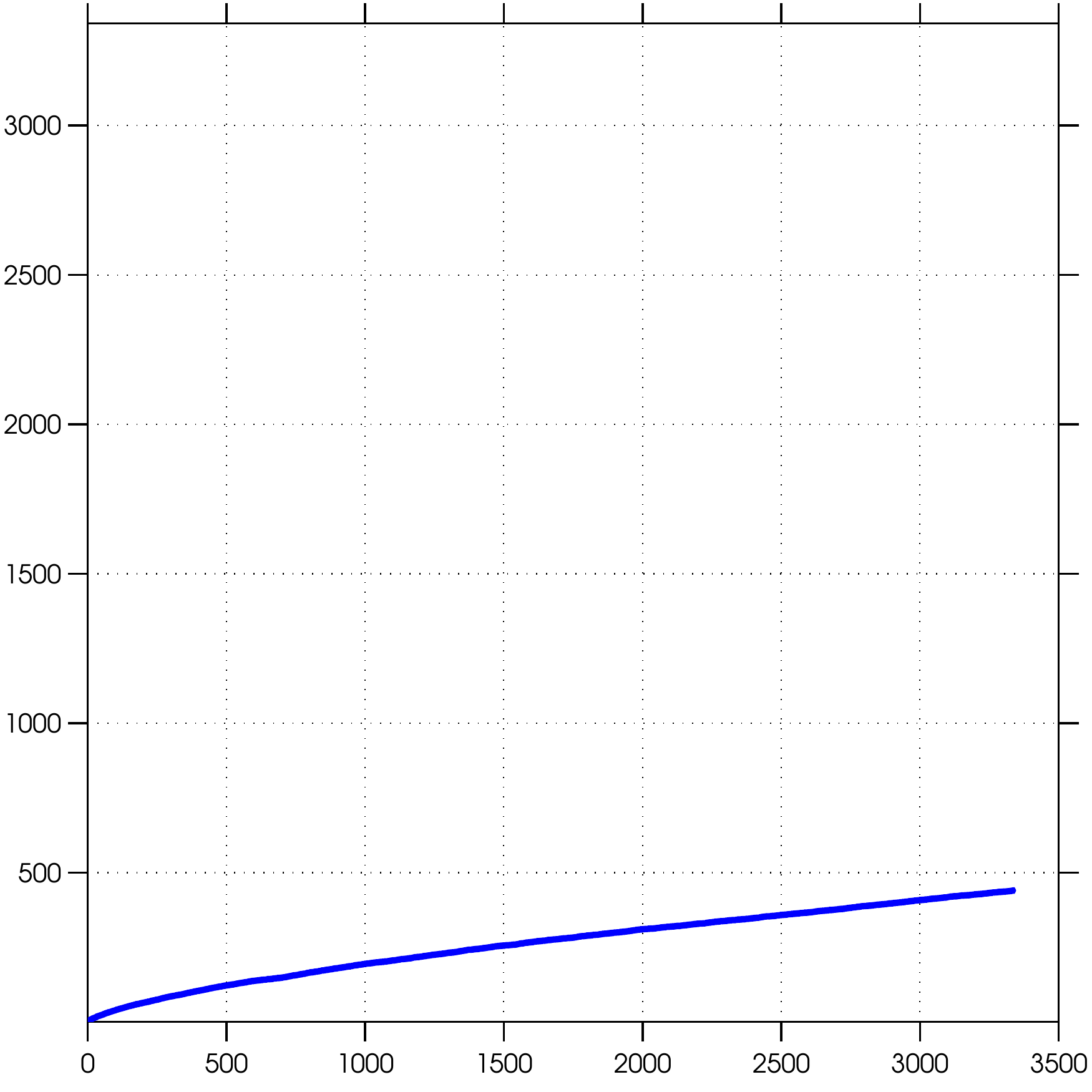}} \subfigure[Statlog]{\includegraphics[scale=0.15]{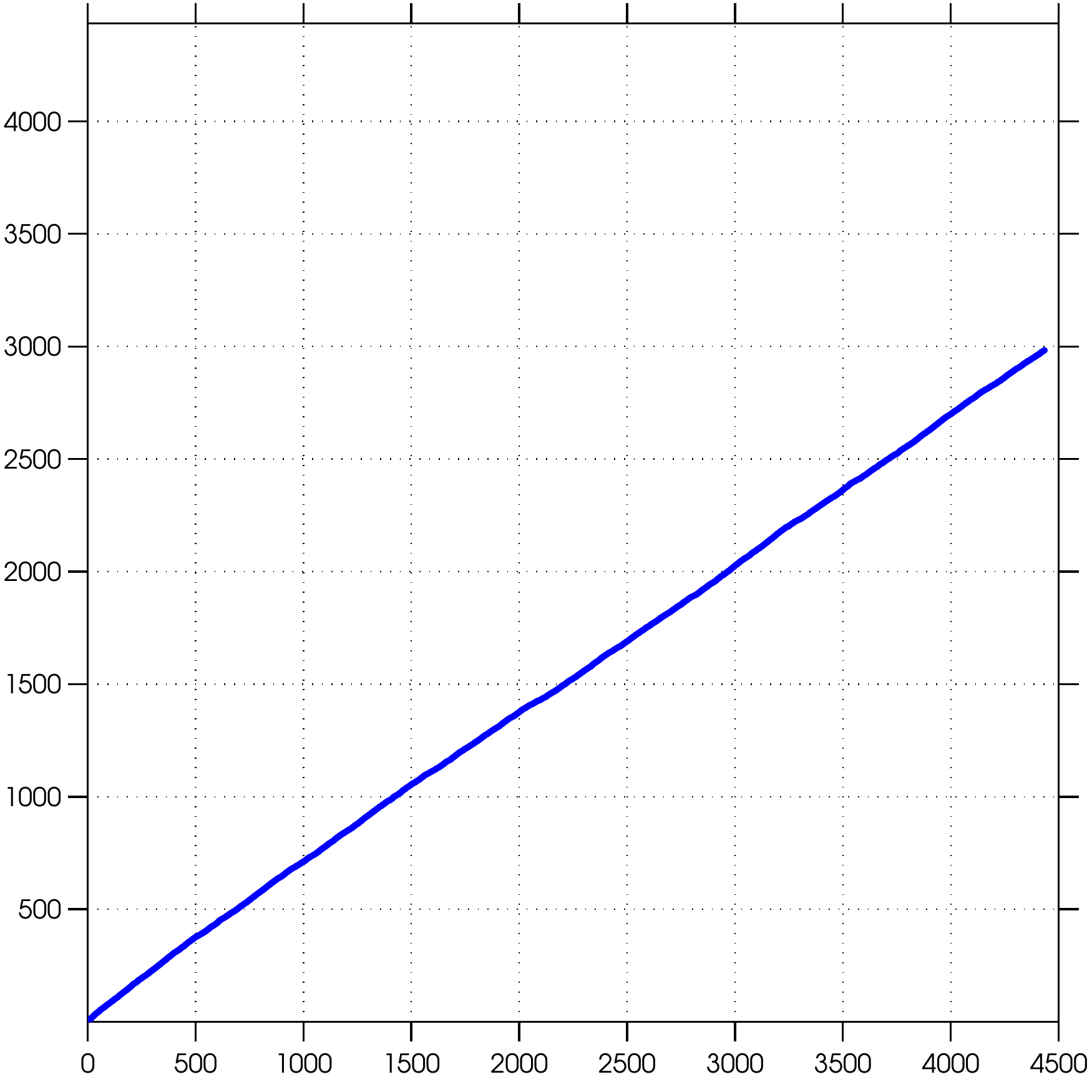}} \subfigure[MNIST]{\includegraphics[scale=0.15]{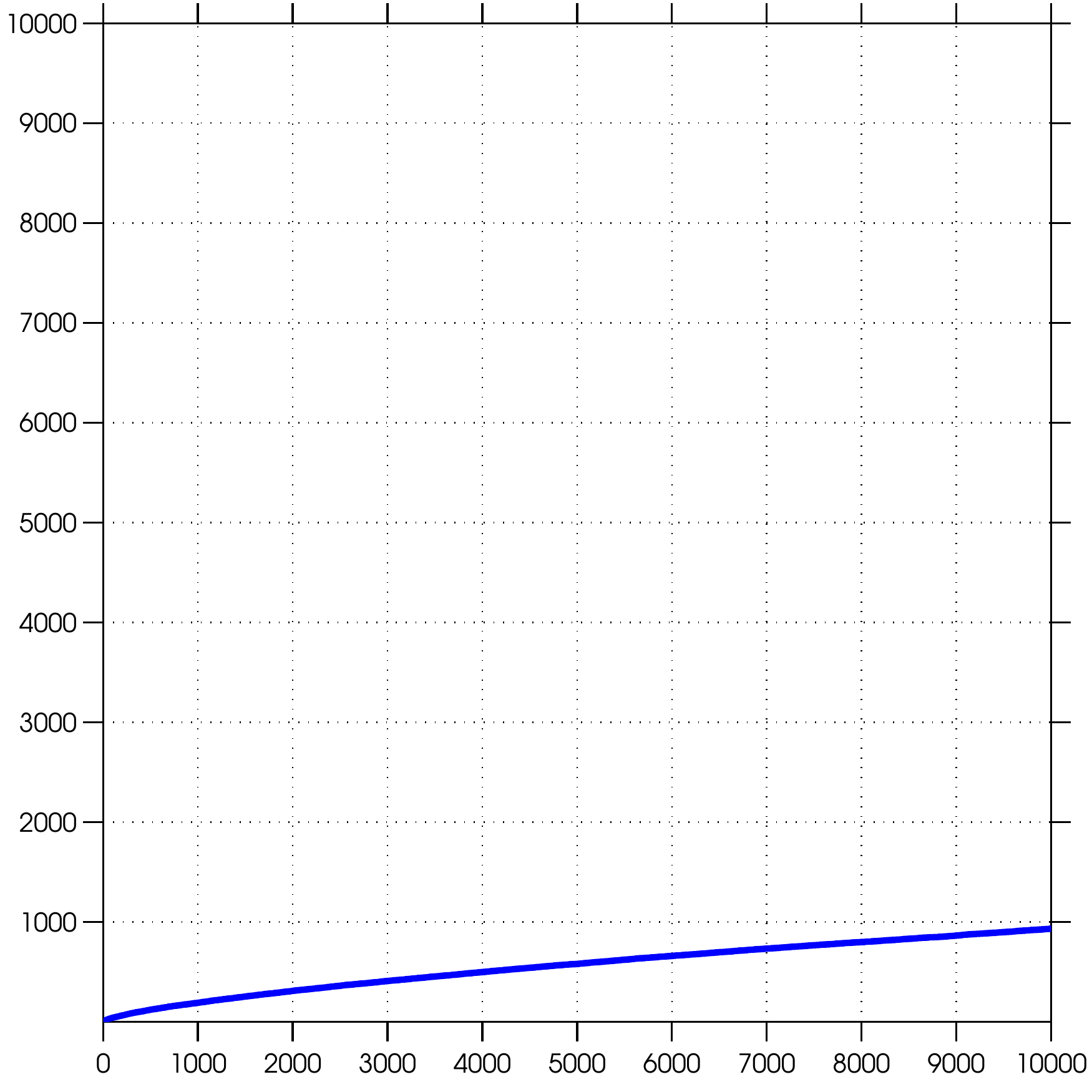}} \subfigure[Whitewine]{\includegraphics[scale=0.15]{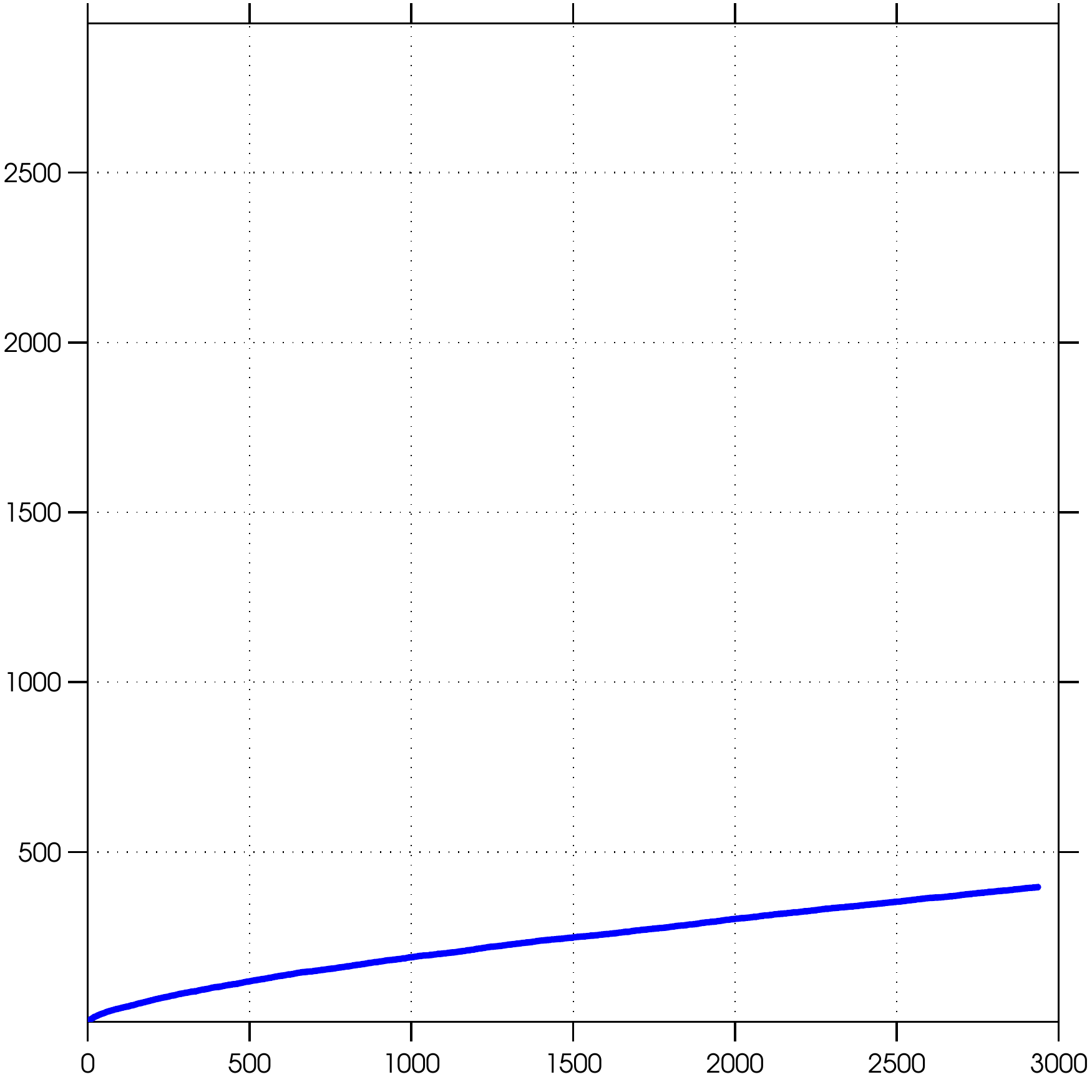}}
\subfigure[Magic]{\includegraphics[scale=0.15]{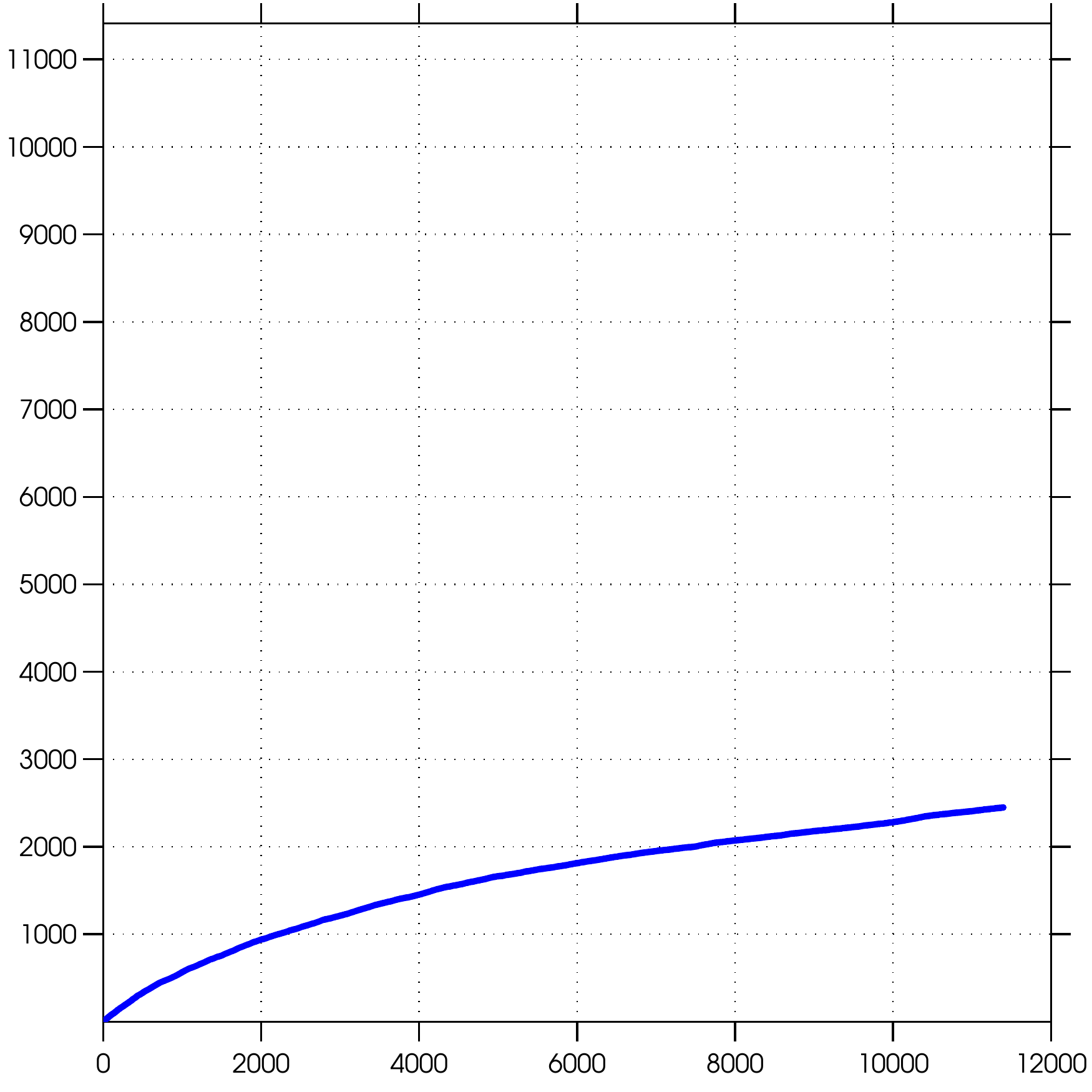}}
\subfigure[Redwine]{\includegraphics[scale=0.15]{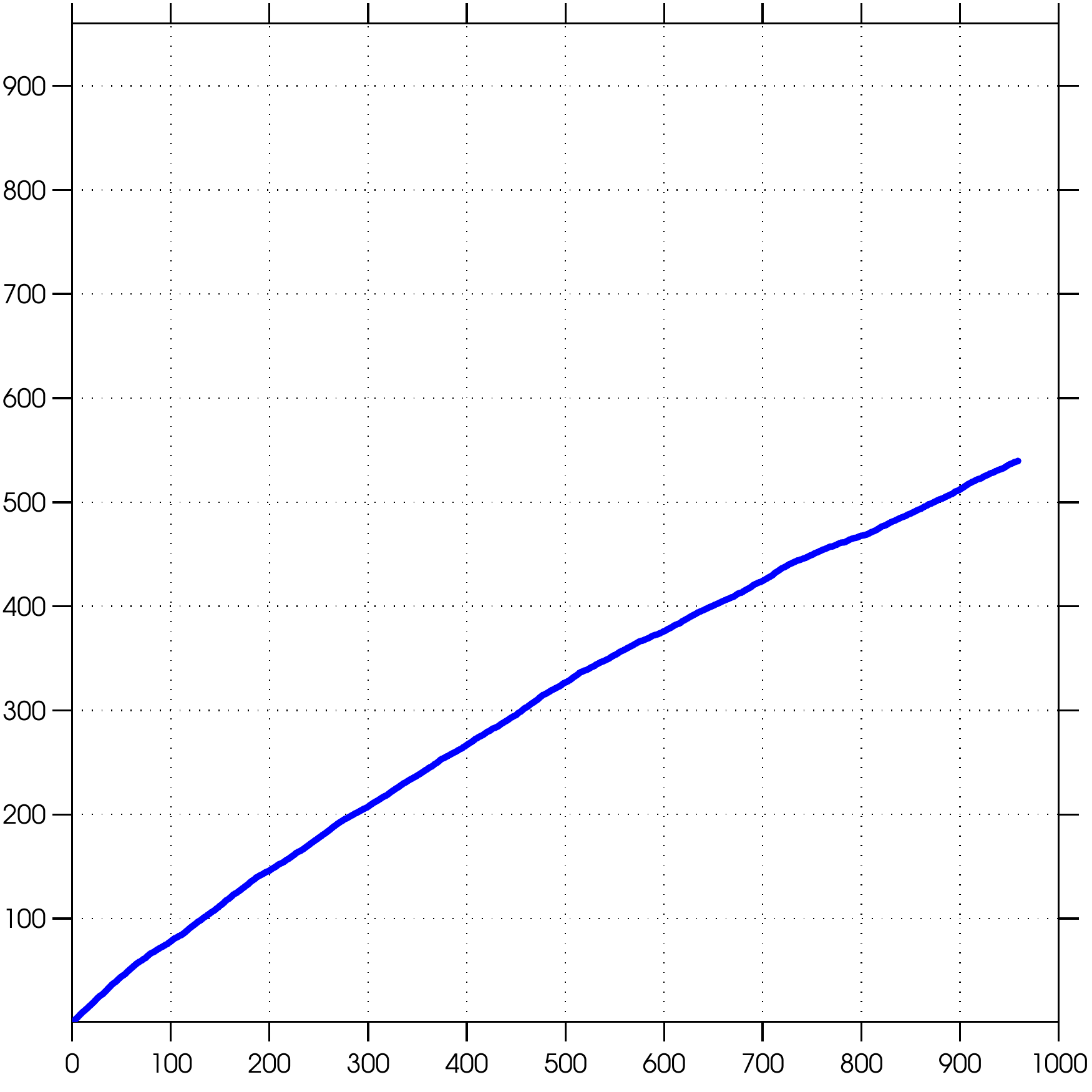}}
\caption{The number of queries made by SMD-AMA as a function of the number of data points seen. \label{fig:queries_points}}
\end{figure*} 
\begin{table*}
  \begin{center}
    \begin{tabular}{|c|c|c|c|c|}
      \hline
      Dataset& Budget& Error rate of SMD-AMA& Error rate of QBB\\
      \hline
      Abalone&441&0.2889&0.3293\\ 
      \hline
      Statlog&2984&0.0491&0.0407\\
      \hline
      MNIST&932&0.1496&0.1756\\
      \hline
      Whitewine&398&0.3075&0.3543\\
      \hline
      Redwine&540&0.2933&0.3146\\
      \hline
      Magic&2450&0.2166&0.2499\\
      \hline
    \end{tabular}
  \end{center}
  \caption{\label{tab:qbb_ama_smd}Comparison of the error rate of QBB and SMD-AMA for a given budget.}
\end{table*}
\subsection{Number of Queries Vs Number of Points Seen}
Figure~\ref{fig:queries_points} shows how the number of queries made by SMD-AMA scale with the number of points seen in the steam on all the four datasets. This scaling is almost linear in the case of Statlog. This was expected, given the fact that on Statlog, we query the labels of almost 65\% of the points in the stream. A similar result holds true even for the Redwine dataset. In contrast, on the Abalone, MNIST,Whitewine, and Magic datasets the scaling seems to be sublinear. 
\subsection{Effect of Parameter $\mu$.}
\label{sec:tradeoff}
\begin{figure}[H]
\centering
 \caption{\label{fig:tradeoff} The plots show, on the MNIST dataset, as a function of the parameter $\mu$, the test error of SMD-AMA, at the end of the stream, and the number of queries made.}
\subfigure[Test Error]
          {
            \includegraphics[scale=0.15]{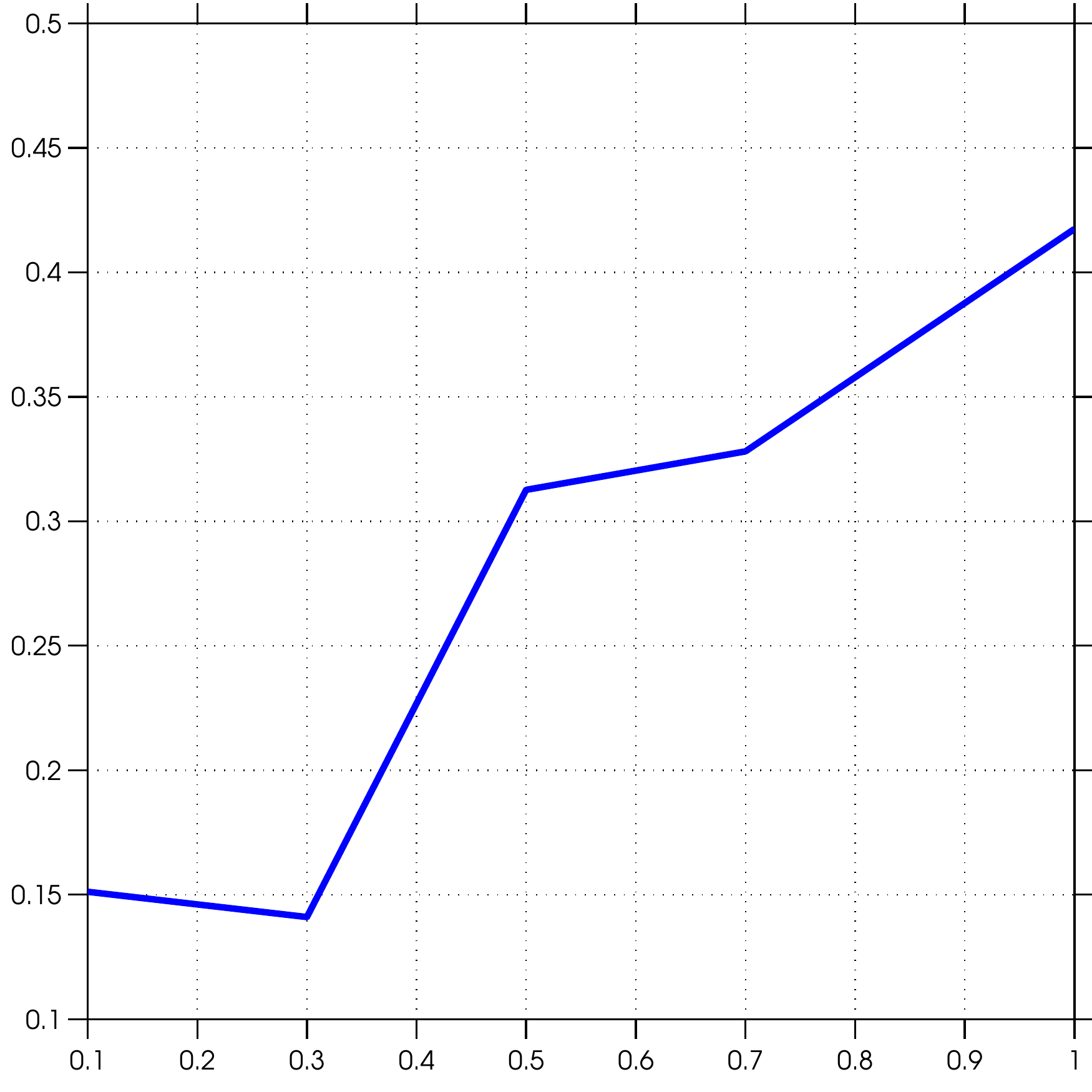}
          }
\subfigure[Queries made]
          {            \includegraphics[scale=0.15]{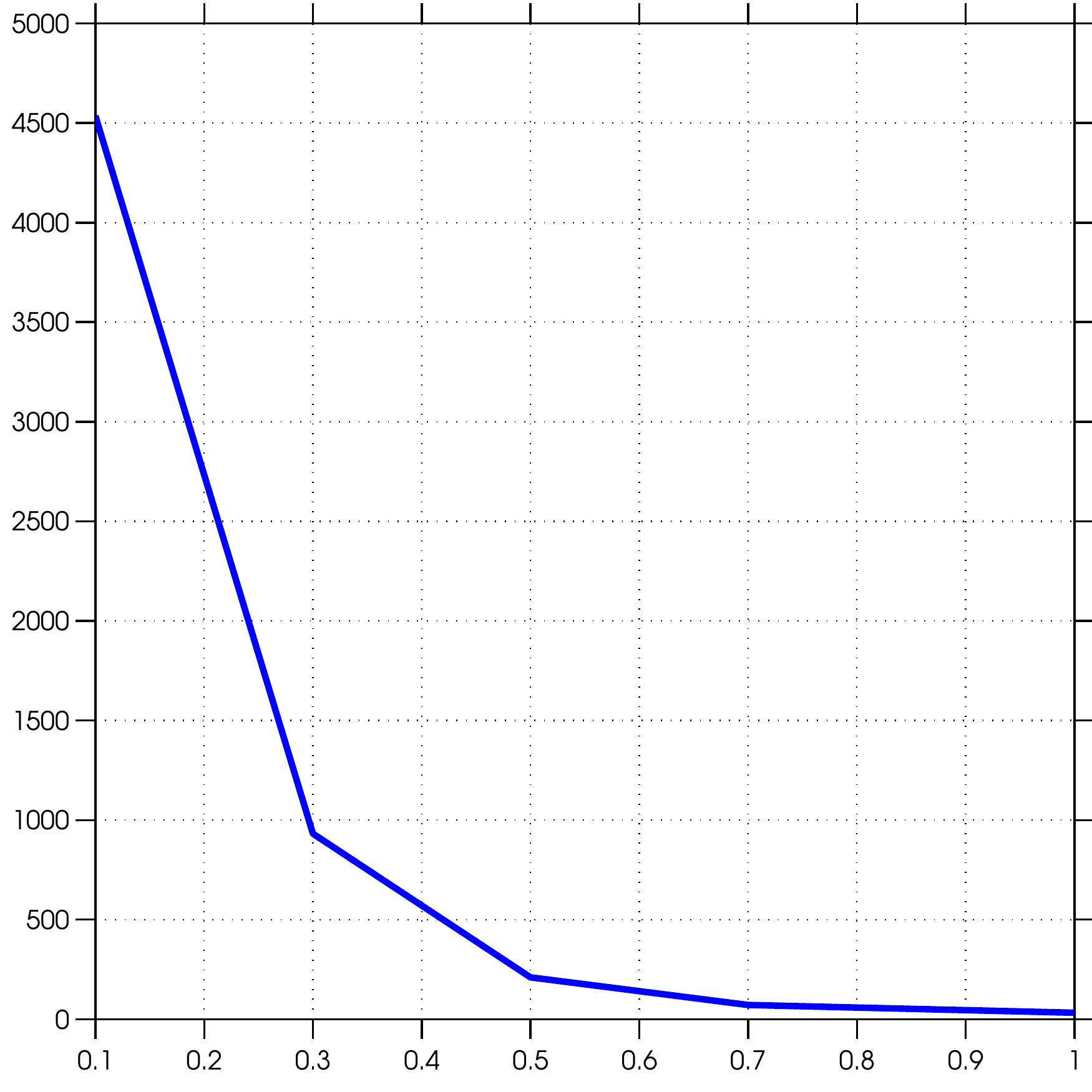}
          }
\end{figure}
From Steps 3, 5 of SMD-AMA, it is clear that $p_t\geq \epsilon_t$. Hence, 
\begin{equation*}
  \bbE[\text{Number of Queries}]\geq\sum_{t=1}^T \epsilon_t=
  \begin{cases}
    \Theta(T^{1-\mu})~\text{if}~\mu<1\\
    \Theta(\log(T))~\text{if}~\mu=1.
  \end{cases}
\end{equation*} 
Theorem~\ref{thm:excess_risk} says that the excess risk of SMD-AMA has an upper bound that scales as $O(\sqrt{T^{\mu-1}})$. Hence from this discussion we know that large $\mu$ leads to a smaller lower bound on the expected number of queries made, but a larger upper bound on the excess risk. Figure~\ref{fig:tradeoff} demonstrates, for the MNIST dataset,  the trade-off between test error and the number of label queries made by SMD-AMA by changing parameter $\mu$. As we can see from these plots, larger $\mu$ lead to small number of queries, but a larger test error.  Similar results have been observed on additional datasets.
\subsection{Comparison with QBB}
In contrast to SMD-PMA and SMD-AMA, QBB is a pool based active learning algorithm, and not a stream based active learning algorithm. Hence, QBB has access to the entire set of unlabeled data points, and in each round, can choose one data point to query. In order to provide a fair comparison of QBB and SMD-AMA, we  used the number of queries made by SMD-AMA from our first set of experiments, as a budget parameter for the QBB algorithm, and report the error rate of the hypothesis returned by SMD-AMA, and QBB at the end of the budget. It is clear from Table~\ref{tab:qbb_ama_smd} that, for all the datasets except Statlog,  QBB is significantly inferior to SMD-AMA. For the Statlog dataset, the test error of SMD-AMA and QBB are comparable, and this is possibly because the budget to our QBB experiments is large. 
\section{Conclusions and Discussions}
In this paper we considered active learning of convex aggregation of classification models. We presented a stochastic mirror descent algorithm, which uses importance weighting to obtain unbiased importance weighted stochastic gradients. We established excess risk guarantees of the resultant convex aggregation and  experimentally demonstrated the good performance of our algorithm. This work can be extended in many directions, some of which are stated below

\textbf{Systematic study of test error-number of queries tradeoff.} As shown in Section~\ref{sec:tradeoff}, parameter $\mu$ captures the tradeoff between the test error and the number of queries made. A systematic study of how the test error, and the number of queries made by SMD-AMA changes with the parameter $\mu$ could help an user to tune the value of $\mu$, in order to be on the ``right point'' of this test error-number of queries curve. A possible direction, towards this goal,  is to establish sharper excess risk guarantees for SMD-AMA and upper bound on the number of queries made by SMD-AMA, as a function of $\mu$ and problem parameters. 
\appendix
\section{Excess Risk Bounds}
In our paper we established an excess risk bound for the convex
aggregate learned by our active learning algorithm.  In order to establish our theorem, we need the following key lemma. This
result was first established in ~\cite{juditsky2005recursive} (See
Proposition 2 in~\cite{juditsky2005recursive}). We are stating the
result only for convenience.
\begin{lemma}
  \label{lem:agg_div}
For any $\theta\in\Theta$, and for any $t\geq 1$, we have
\begin{equation}
  \sum_{t=1}^T \la \theta_{t-1}-\theta,\nabla\cR(\theta_{t-1})\leq \beta_T\cR(\theta)-\sum_{t=1}^T \la \theta_{t-1}-\theta_t,\Delta_{t}(\theta_{t-1})\ra+\sum_{t=1}^T \frac{||g_t(\theta_{t-1})||_{\infty}^2}{2\beta_{t-1}}\nonumber.
\end{equation}
\end{lemma}
\begin{proof}
  Let 
\begin{equation}
  \cR(\theta)=
  \begin{cases}
    \log(M)+\sum_{j=1}^M \theta_j\log(\theta_j),~\text{if}~ \theta\in \Delta_M\\
    \infty~\text{otherwise}.
  \end{cases}
\end{equation}
$\cR(\theta)$ is the normalized Entropy function that has been widely used in the mirror-descent algorithm, when optimizing over a simplex. The $\beta$ Fenchel dual of $\cR$ is
\begin{equation}
  \Rbetastar(z)=\beta\log\left(\frac{1}{M}\sum_{j=1}^M e^{-z_j/\beta}\right)
\end{equation}
Now, since $\Rbetastar$ is continously differentiable, we get 
\begin{equation}
  \label{eqn:taylor}
  \cR_{\beta_{t-1}}^{*}(\xi_t)=\cR_{\beta_{t-1}}^{*}(\xi_{t-1})+\int_0^1 \la\xi_t-\xi_{t-1},\nabla\cR_{\beta_{t-1}}^{*}(\tau\xi_{t}+(1-\tau)\xi_{t-1})\ra~\mathrm{d}\tau
\end{equation}
Since, $\cR$ is 1-strongly convex w.r.t. the $||\cdot||_1$ norm, the $\beta-$ Fenchel dual of $\cR$ is continously differentiable and satisfies, for each pair $z,\tilde{z}$ in the domain of $\cR_{\beta}$
\begin{equation}
  \label{eqn:lip}
  ||\nabla\cR_{\beta}^{*}(z)-\nabla\cR_{\beta}^{*}(\tilde{z})||_1\leq \frac{1}{\beta} ||z-\tilde{z}||_{\infty}. 
\end{equation}
By definition, $\xi_t=\xi_{t-1}+g_t(\theta_{t-1})$. Substituting this in Equation~\ref{eqn:taylor}, we get
\begin{align}
  \cR_{\beta_{t-1}}^{*}(\xi_t)&=\cR_{\beta_{t-1}}^{*}(\xi_{t-1})+\int_0^1\la g_t(\theta_{t-1}),\nabla \cR_{\beta_{t-1}}^{*}(\tau\xi_t+(1-\tau)\xi_{t-1})\ra~\mathrm{d}\tau\nonumber\\
  \begin{split}
    &=\cR_{\beta_{t-1}}^{*}(\xi_{t-1})+\int_0^1\la g_t(\theta_{t-1}),\nabla\cR_{\beta_{t-1}}^{*}(\tau\xi_t+(1-\tau)\xi_{t-1})-\nabla\cR_{\beta_{t-1}}^{*}(\xi_{t-1})\ra~\mathrm{d}\tau\nonumber\\
    &\qquad+\int_0^1 \la g_t(\theta_{t-1}),\nabla\cR_{\beta_{t-1}}^{*}(\xi_{t-1})\ra~\mathrm{d}\tau
\end{split}\nonumber\\
  &\leq \cR^{*}_{\beta_{t-1}}(\xi_{t-1})+\frac{1}{2\beta_{t-1}}||g_t(\theta_{t-1})||_{\infty}^2+\la g_t(\theta_{t-1})\nabla\cR^{*}_{\beta_{t-1}}(\xi_{t-1})\ra\label{eqn:seq},
\end{align}
where in the last step we applied H\"older's inequality.
By design, our sequence $\beta_0,\beta_1,\ldots$ is an increasing sequence. Given $\beta$, $\cR_{\beta}^{*}$ is a non-increasing function. Hence, by Equation~\ref{eqn:seq}, we get
\begin{equation}
  \label{eqn:monot}
  \cR_{\beta_t}^{*}(\xi_t)\leq \cR_{\beta_{t-1}}^{*}(\xi_t)\leq \cR^{*}_{\beta_{t-1}}(\xi_{t-1})+\la g_t(\theta_{t-1}),\nabla\cR^{*}_{\beta_{t-1}}(\xi_{t-1})\ra+\frac{1}{2\beta_{t-1}}||g_t(\theta_{t-1})||_{\infty}^2
\end{equation}
Summing over $t\geq 1$, and telescoping the sum,  we get 
\begin{equation}
  \cR^{*}_{\beta_T}(\xi_T)-\cR^{*}_{\beta_0}(\xi_0)\leq \sum_{t=1}^T \la g_t(\theta_{t-1}),\nabla\cR_{\beta_{t-1}}(\xi_{t-1})\ra+\sum_{t=1}^T \frac{1}{2\beta_{t-1}}||g_t(\theta_{t-1})||_{\infty}^2
\end{equation}
Since, by definition of the mirror-descent procedure, $\nabla\cR^{*}_{\beta_{t-1}}(\xi_{t-1})=-\theta_{t-1}$, hence we get 
\begin{equation}
  \cR_{\beta_T}^{*}(\xi_T)-\cR_{\beta_0}^{*}(\xi_0)\leq -\sum_{t=1}^T \la g_t(\theta_{t-1}),\theta_{t-1}\ra+\sum_{t=1}^T \frac{1}{2\beta_{t-1}}||g_t(\theta_{t-1})||_{\infty}^2.
\end{equation}
Rearranging, and using the fact that $\xi_T=\sum_{t=1}^T g_t(\theta_{t-1})$, we get that for any $\theta\in\Delta_M$, 
\begin{equation}
  \label{eqn:temp1}
  \sum_{t=1}^T \la \theta_{t-1}-\theta,g_t(\theta_{t-1})\ra\leq \cR^{*}_{\beta_0}(0)-\cR^{*}_{\beta_T}(\xi_T)-\la\xi_T,\theta\ra+\sum_{t=1}^T\frac{||g_t(\theta_{t-1})||_{\infty}^2}{2\beta_{t-1}}.
\end{equation}
Let, $\Delta_{t}(\theta_{t-1})\defeq g_t(\theta_{t-1})-\nabla R(\theta_{t-1})$.  By definition, $\bbE \Delta_{t}(\theta_{t-1})=0$. Replacing, for $\Delta_{t}(\theta_{t-1})$ in Equation~\ref{eqn:temp1}, we get 
\begin{equation}
  \label{eqn:temp2}
  \sum_{t=1}^T \la \theta_{t-1}-\theta,\nabla R(\theta_{t-1})+\Delta_t(\theta_{t-1})\ra \leq \cR_{\beta_0}^{*}(0)-\cR_{\beta_T}^{*}(\xi_T)-\la \xi_T,\theta\ra+\sum_{t=1}^T \frac{||g_t(\theta_{t-1})||_{\infty}^2}{2\beta_{t-1}}
\end{equation} 
\end{proof}
Sine, $\cR^{*}_{\beta}$ is the $\beta$- Fenchel dual of $\cR$, and $\xi_T\in\bbE^{*}$, i.e. 
\begin{equation}
  \beta\cR(\theta)=\sup_{z\in E^{*}}[-z^T\theta-\cR^{*}_{\beta}(z)], 
\end{equation}
this means that  
\begin{equation}
  \label{eqn:temp3}
  -\la\theta,\xi_T\ra -\cR^{*}_{\beta_T}(\xi_T)\leq \beta_T \cR(\theta) 
\end{equation}
Putting together Equations~(\ref{eqn:temp2},~\ref{eqn:temp3}),  we get 
\begin{align}
  \sum_{t=1}^T \la \theta_{t-1}-\theta,\nabla\cR(\theta_{t-1})\ra&\leq \cR_{\beta_0}(0)+\beta_T\cR(\theta)-\sum_{t=1}^T \la \theta_{t-1}-\theta_t, \Delta_t(\theta_{t-1})\ra+\sum_{t=1}^T \frac{||g_t(\theta_{t-1}||_{\infty}^2}{2\beta_{t-1}}\nonumber\\
  &\leq \beta_T\cR(\theta)-\sum_{t=1}^T \la \theta_{t-1}-\theta_t,\Delta_{t}(\theta_{t-1})\ra+\sum_{t=1}^T \frac{||g_t(\theta_{t-1})||_{\infty}^2}{2\beta_{t-1}}\nonumber.
\end{align}
This completes our proof.
\bibliographystyle{plainnat}
\footnotesize{\bibliography{agg.bib}}
\end{document}